\documentclass{article}
\usepackage{iclr2027_conference,times}

\usepackage{amsmath,amsfonts,bm}

\def\eqref#1{\textup{(\ref{#1})}}

\def\1{\bm{1}}

\DeclareMathAlphabet{\mathsfit}{\encodingdefault}{\sfdefault}{m}{sl}
\SetMathAlphabet{\mathsfit}{bold}{\encodingdefault}{\sfdefault}{bx}{n}

\newcommand{\E}{\mathbb{E}}

\usepackage{amsmath,amssymb,amsthm,mathtools}
\usepackage{booktabs}
\usepackage{array}
\usepackage{multirow}
\usepackage{graphicx}
\usepackage{flafter}
\usepackage{placeins}
\usepackage{wrapfig}
\usepackage{microtype}
\usepackage{enumitem}
\usepackage{xcolor}
\usepackage{hyperref}
\usepackage{url}
\usepackage{tikz}
\newcommand{\pending}{\textnormal{\textsc{tbd}}}
\newcommand{\plotslot}[4]{%
  \IfFileExists{figures/#1.pdf}{\includegraphics[width=\linewidth]{figures/#1.pdf}}{%
  \resizebox{\linewidth}{!}{%
  \begin{tikzpicture}[x=1cm,y=1cm]
  \fill[black!2] (0,0) rectangle (4.25,2.25);
  \draw[black!12] (0,0.75)--(4.25,0.75) (0,1.5)--(4.25,1.5);
  \draw[black!60,->] (0,0)--(4.45,0);
  \draw[black!60,->] (0,0)--(0,2.4);
  \node[font=\scriptsize,text=black!55,align=center] at (2.13,1.13) {Awaiting measurements};
  \node[font=\small,anchor=north] at (2.13,-0.15) {#3};
  \node[font=\small,rotate=90,anchor=south] at (-0.12,1.12) {#4};
  \node[font=\small,anchor=south,align=center] at (2.13,2.45) {#2};
  \end{tikzpicture}}}}
\hypersetup{colorlinks=false}

\newtheorem{definition}{Definition}
\newtheorem{theorem}{Theorem}
\newtheorem{proposition}{Proposition}

\newtheorem{corollary}{Corollary}

\newcommand{\M}{\mathcal{M}}
\newcommand{\Mh}{\widehat{\mathcal{M}}}
\newcommand{\PiSet}{\Pi}
\newcommand{\TV}{\operatorname{TV}}

\newcommand{\Jh}{\widehat J}

\newcommand{\epsint}{\varepsilon_{\mathrm I}}
\newcommand{\Prob}{\mathbb{P}}
\newcommand{\ind}{\mathbb{I}}

\newcommand{\FW}{\mathcal{F}_{\mathrm W}}
\newcommand{\FA}{\mathcal{F}_{\mathrm A}}
\DeclareMathOperator{\spn}{span}

\title{Dual-Frontier: When Can an Agent Trust Its World Model?}

\author{%
\resizebox{\dimexpr\textwidth-2\tabcolsep\relax}{!}{\bfseries
Huatai Zhu$^{1}$\thanks{Equal contribution.}, Qiang Chen$^{2}$\footnotemark[1],
Ziqian Kou$^{3}$, Wenhao Li$^{4}$, Fei Wang$^{5}$, Yichao Cao$^{1}$,
Xiu Su$^{1}$\thanks{Corresponding authors.}, Yi Chen$^{2}$\footnotemark[2]}\\[4pt]
\normalfont\small
\makebox[\dimexpr\textwidth-2\tabcolsep\relax][c]{$^{1}$Central South University \quad
$^{2}$The Hong Kong University of Science and Technology}\\
\makebox[\dimexpr\textwidth-2\tabcolsep\relax][c]{$^{3}$Xiangjiang Laboratory \quad
$^{4}$University of Sydney}\\
\makebox[\dimexpr\textwidth-2\tabcolsep\relax][c]{$^{5}$University of Science and Technology of China}
}

\iclrfinalcopy

\begin{document}
\maketitle
\lhead{}
\begingroup
\renewcommand{\thefootnote}{}
\footnotetext{Paper Under Review.}
\endgroup
\vspace{-12pt}

\begin{abstract}
Learned world models are becoming essential to general-purpose agents: by predicting action consequences, they support planning and decision-making while reducing reliance on costly trial and error. This reliance creates a fundamental ambiguity: when a world-model-guided decision fails, the trajectory alone may not reveal whether the agent's decision rule or the world model caused the loss. We formalize this \emph{failure-attribution problem} as a counterfactual decomposition of return loss and prove that its components are not identifiable from passive interaction, even for finite-horizon planners. This obstruction motivates \emph{Dual-Frontier}, a learning principle that admits a world-model-guided decision only when its predicted advantage exceeds a certified bound on decision-relevant world-model error; otherwise, evidence is allocated to world-model verification. Action-conditioned value bounds and a closed-loop extension guarantee non-decreasing return for admitted decisions. Calibrated gates and simultaneous confidence sequences support adaptive evidence reuse, with sufficient and necessary verification bounds. Controlled learned-model experiments validate the predicted failure modes and certification behavior, while cross-backbone tool-use benchmarks instantiate the same verify-then-promote rule in realistic agent world-model pipelines, consistently improving decision quality and reliability.
\end{abstract}

\section{Introduction}

World models are becoming core infrastructure for agents that act beyond their immediate observations. Consequential action requires anticipating how choices alter future states, information, and opportunities: general multi-step agency entails recoverable knowledge of environmental dynamics \citep{richens2025general}, even under partial observability and stochasticity \citep{cifuentes2026general}. Learned world models operationalize this knowledge as action-conditioned predictors rolled forward before execution. Whether expressed in latent states, video, language, or internal dynamics, they augment the present observation with forecasts used to compare actions, plan farther ahead, and reduce costly online trial and error.

This predictive interface now spans latent-imagination control in DreamerV3 \citep{hafner2025dreamer} and scalable model-predictive control in TD-MPC2 \citep{hansen2024tdmpc2}, as well as navigation \citep{bar2025nwm}, manipulation \citep{assran2025vjepa}, driving \citep{russell2025gaia}, and web interaction \citep{chae2025web}. It is increasingly adaptive: WorldEvolver revises predictive memory at test time \citep{zhang2026worldevolver}, CoMAP alternates world-model adaptation with agent reflection \citep{liu2026comap}, and recent systems co-train predictive knowledge and policies \citep{lu2026paw} or co-evolve simulators with agents \citep{guo2026genenv}. Reliability is thus part of the decision mechanism. Short rollouts limit model exploitation \citep{janner2019mbpo}; horizon-calibrated uncertainty addresses compounding error \citep{wan2026hauwm}; safe-improvement methods constrain policy changes \citep{delgrange2026deepspi}.

Yet a poor world-model-guided outcome poses an unresolved question: \emph{what should improve next---the agent's decision rule or the world model on which it relied?} The same trajectory can arise because an inadequate rule ignored an accurate forecast or because an inaccurate forecast misled an otherwise sound rule. Passive interaction records only their composition. We call separating these causes \emph{failure attribution}. The distinction is operational: learning against a misleading world model can reinforce a bad decision, while gathering more world-model data after the relevant forecast is adequate wastes evidence. Nor can global prediction accuracy decide the issue. Perceptual quality and closed-loop success can diverge \citep{zhang2026worldinworld}; an arbitrarily small transition error may reverse nearly tied actions, whereas a large error in an irrelevant coordinate may alter none. The missing object is decision-specific: does current evidence establish that a world-model-proposed behavior improves upon an explicit reference?

To solve this attribution-and-trust problem, we develop \textbf{Dual-Frontier}, a theory of decision-specific trust for learned-world-model agents. It evaluates the true-return contrast between a candidate behavior proposed with the world model and a reference behavior. The world-model frontier retains promising comparisons lacking evidence; the agent frontier admits those remaining beneficial after world-model and estimation uncertainty. This \emph{verify-then-promote} rule directs unresolved comparisons to targeted verification and certified ones to agent improvement, without interpreting rejection as proof of model failure. Our contributions are:
\begin{itemize}[leftmargin=*,itemsep=1.5pt,topsep=3pt]
    \item \textbf{Failure attribution.} A fixed-operator counterfactual decomposition separates agent deficiency from world-model effect; a two-step construction proves passive non-identifiability for exact and Monte Carlo planners.
    \item \textbf{Decision reliability.} An exact Bellman-residual identity converts action-conditioned world-model error into comparison-specific radii, sharp separations, and a closed-loop condition guaranteeing nonnegative expected improvement.
    \item \textbf{Dual-frontier learning.} Calibrated promotion rules reuse one simultaneous certificate under adaptive world-model and request selection, with progress and matching-order sufficient/necessary evidence bounds.
    \item \textbf{Empirical validation.} Learned-model experiments test non-identifiability, harmful imagined improvements, and qualification; matched agent--world-model evaluations separate reliability, realized quality, and verification cost.
\end{itemize}

\section{Related Work}

\paragraph{General-purpose agents in interactive environments.}
Modern agents transact with websites \citep{zhou2023webarena}, operate desktops \citep{xie2024osworld}, repair repositories \citep{jimenez2024swebench}, and act in embodied environments \citep{assran2025vjepa}. AgentBench spans eight interactive settings \citep{liu2023agentbench}, while continual agents face changing objectives \citep{liu2025online}. Across domains, multi-step attainment requires recoverable predictive knowledge \citep{richens2025general}, including under partial observability and stochasticity \citep{cifuentes2026general,huang2026toolscontinuousflowevolving}. Learned world models expose such knowledge for planning, from navigation \citep{bar2025nwm} to web interaction \citep{chae2025web}. Prior work primarily measures end-to-end success or builds domain-specific predictors, leaving forecasts embedded in the system. Dual-Frontier instead isolates the world-model--decision-rule interface and asks whether evidence supports one proposed behavior comparison.

\paragraph{Learned world models for agent planning and adaptation.}
We use \emph{learned world model} for an action-conditioned predictor that compares future courses of action, whether latent, visual, textual, or internalized. The lineage extends from recurrent simulators \citep{ha2018world} and value-equivalent models \citep{schrittwieser2020muzero} to DreamerV3 \citep{hafner2025dreamer}, TD-MPC2 \citep{hansen2024tdmpc2}, and multi-task policy learning \citep{georgiev2025pwm}. Current systems forecast navigation \citep{yao2025navmorph}, manipulation \citep{assran2025vjepa}, driving \citep{russell2025gaia}, and web transitions \citep{chae2025web}. WorldEvolver adapts predictive memory online \citep{zhang2026worldevolver}; CoEx updates persistent beliefs during exploration \citep{kim2025coex}. WebEvolver combines synthetic trajectories with look-ahead planning \citep{fang2025webevolver}, while CoMAP alternates world-model adaptation and reflection \citep{liu2026comap}. PaW co-trains policy and world model \citep{lu2026paw}; GenEnv co-evolves agents and simulators \citep{guo2026genenv}; DreamGym and Agent World Model scale synthesized interaction \citep{chen2026dreamgym,wang2026awm}. These systems optimize or exploit prediction to improve the agent, but do not identify whether a failed decision implicates its rule or its forecast. Dual-Frontier formalizes that ambiguity and qualifies a behavior comparison rather than a world model globally.

\paragraph{Reliable model-based learning and adaptation.}
World-model exploitation motivates short rollouts \citep{janner2019mbpo} and pessimism \citep{yu2020mopo}; policy-aware learning targets downstream gradient error \citep{abachi2020paml,doro2020gradient}. Recent work further studies horizon-dependent uncertainty \citep{wan2026hauwm}, local safe improvement \citep{delgrange2026deepspi}, and confidence-filtered foresight \citep{zhang2026worldevolver}. WAKER collects data where estimated world-model error is high \citep{rigter2024waker}; AdaWM separates dynamics and policy mismatch under transfer, then fine-tunes the indicated component \citep{wang2025adawm}. These methods mainly limit error or respond to an assumed diagnostic. By contrast, Dual-Frontier first proves that passive failure need not identify its source, then links action-conditioned error to value and imagined-gradient distortion. WAKER targets accuracy across environments, whereas we certify a behavior comparison; AdaWM selects adaptation under shift, whereas we establish when a world-model-guided change is actually justified in practice. Split conformal calibration \citep{angelopoulos2023conformal} and simultaneous confidence yield certificates valid under adaptive evidence reuse, connecting attribution, decision reliability, and allocation in one formulation.

\section{Set up}
\label{sec:setup}

For task $z$, let $\M_z=(\mathcal S,\mathcal A,P_z,r_z,\rho_z,H)$ and
$\Mh_z=(\mathcal S,\mathcal A,\widehat P_z,\widehat r_z,\rho_z,H)$ denote the true world
and its learned world model. The spaces are standard Borel, $H\geq1$, rewards lie in
$[0,R_b]$, and both systems share an information interface
(Appendix~\ref{app:interface}). Suppressing $z$, let $J,\Jh$ be undiscounted returns,
$\mu_t^\pi$ the true law of $(s_t,a_t)$, and
$\TV(P,Q)=\sup_B|P(B)-Q(B)|$.

A fixed operator $\mathsf A_\phi$ maps a supplied world to a policy, with $\phi$
fixing its search, information access, budget, and randomization. Replacing only the
world defines
\begin{equation}
\widehat\pi=\mathsf A_\phi(\Mh),
\qquad
\pi^\circ=\mathsf A_\phi(\M).
\label{eq:counterfactual-policies}
\end{equation}
Let $J^*=\sup_{\pi\in\PiSet}J(\pi)$, where $\PiSet$ contains both outputs.

\begin{definition}[Counterfactual attribution]\label{def:decomp}
Let
\begin{equation}
\begin{gathered}
A:=J^*-J(\pi^\circ),
\qquad
W:=J(\pi^\circ)-J(\widehat\pi),\\
R:=J^*-J(\widehat\pi)
=[J^*-J(\pi^\circ)]+[J(\pi^\circ)-J(\widehat\pi)]
=A+W.
\end{gathered}
\label{eq:failure-decomposition}
\end{equation}
Here $R,A\geq0$, whereas $W$ is a signed world-model effect.
\end{definition}

The decomposition is relative to one fixed predictive decision procedure; because
model error may accidentally help a limited agent, $W$ is signed.

\begin{theorem}[Passive non-identifiability]\label{thm:impossibility}
For every $c\in(0,R_b]$, a fixed known operator and supplied world model admit a family
of two-step true worlds, indexed by $\lambda\in[0,c]$, with
$(A,W)=(\lambda,c-\lambda)$ and identical laws of arbitrarily many deployed episodes.
Any attribution estimator based on these episodes, the known operator, and the supplied
world model satisfies
\begin{equation}
\begin{aligned}
\sup_{\lambda\in[0,c]}\E_\lambda|\widehat A-\lambda|
&\geq\frac12\left(\E_0|\widehat A|+\E_c|\widehat A-c|\right)\\
&=\frac12\E\left[|\widehat A|+|\widehat A-c|\right]
\geq\frac c2.
\end{aligned}
\label{eq:main-attribution-risk}
\end{equation}
where the middle expectation is under the common observational law. Endpoint
classification has minimax error $1/2$. Both bounds are attained, for exact and
finite-sample Monte Carlo planning.
\end{theorem}

Appendix~\ref{app:attribution} gives the construction and minimax proof. Notably,
the supplied world model may be exact on the deployed occupancy; the ambiguity can
reside entirely in untried actions. The obstruction is therefore passive rather than
a prohibition on verification: intervention can recover missing information, whereas
an unstable operator can amplify small predictive error
(Proposition~\ref{prop:operator}). We therefore certify a specified next use rather
than infer ownership from passive failure.

\section{World-Model Reliability}
\label{sec:reliability}

\subsection{Certifying a proposed use}

A request $x=(z,\Mh,\pi_0,\pi_1)$ specifies a learned world model, a reference
behavior $\pi_0$, and a candidate $\pi_1$, including their continuation rules. Define
\begin{equation}
\Gamma_{\M}(x)=J_{\M}(\pi_1)-J_{\M}(\pi_0),\qquad
\ell_{\mathcal C}(x)=\inf_{\mathcal N\in\mathcal C}\Gamma_{\mathcal N}(x).
\label{eq:use-contrast}
\end{equation}
On $\M\in\mathcal C$, $\ell_{\mathcal C}>0$ certifies improvement; failure to certify
does not imply harm. The policies remain fixed while the evaluation world varies
(Appendix~\ref{app:contrast}).

For bounded $f$, write $Pf=\int f(s')P(\mathrm ds'\mid s,a)$ and
$\spn(f)=\sup f-\inf f$. With learned value $\widehat V_t^\pi$ and
$\widehat V_H^\pi=0$, set
\begin{align}
\Delta_t^\pi&=r-\widehat r+(P-\widehat P)\widehat V_{t+1}^\pi,
\label{eq:bellman-residual}\\
\epsint(\pi)&=\sum_{t=0}^{H-1}\E_{\mu_t^\pi}
 [|r-\widehat r|+\spn(\widehat V_{t+1}^\pi)\TV(P,\widehat P)].
\label{eq:intervention-error}
\end{align}

\begin{theorem}[Decision reliability]\label{thm:planning}
For every policy and every policy comparison,
\begin{equation}
\begin{aligned}
J(\pi)-\Jh(\pi)
&=\int_{\mathcal S}\!\bigl(V_0^\pi-\widehat V_0^\pi\bigr)(s)\,\rho(\mathrm ds)\\
&=\sum_{t=0}^{H-1}
  \int_{\mathcal S\times\mathcal A}\!\Delta_t^\pi(s,a)\,
  \mu_t^\pi(\mathrm ds,\mathrm da)\\
&=\sum_{t=0}^{H-1}\E_{\mu_t^\pi}\Delta_t^\pi.
\end{aligned}
\label{eq:value-bound}
\end{equation}
Moreover,
\begin{equation}
\begin{aligned}
|J(\pi)-\Jh(\pi)|
&\leq\sum_{t=0}^{H-1}\int
 \left(|r-\widehat r|+
 \left|\int\widehat V_{t+1}^\pi\,\mathrm d(P-\widehat P)\right|\right)
 \mathrm d\mu_t^\pi\\
&\leq\sum_{t=0}^{H-1}\int
 \left(|r-\widehat r|+
 \spn(\widehat V_{t+1}^\pi)\TV(P,\widehat P)\right)
 \mathrm d\mu_t^\pi\\
&=\epsint(\pi).
\end{aligned}
\label{eq:value-error}
\end{equation}
Consequently,
\begin{equation}
\begin{aligned}
J(\pi_1)-J(\pi_0)
&=\Jh(\pi_1)-\Jh(\pi_0)
 +[J(\pi_1)-\Jh(\pi_1)]-[J(\pi_0)-\Jh(\pi_0)]\\
&\geq\Jh(\pi_1)-\Jh(\pi_0)
 -|J(\pi_1)-\Jh(\pi_1)|-|J(\pi_0)-\Jh(\pi_0)|\\
&\geq\Jh(\pi_1)-\Jh(\pi_0)
 -\epsint(\pi_1)-\epsint(\pi_0).
\end{aligned}
\label{eq:comparison}
\end{equation}
\end{theorem}

\begin{corollary}[Planning regret]\label{cor:planning-regret}
If $\Jh(\widehat\pi)\geq\sup_{\pi\in\PiSet}\Jh(\pi)-\delta_p$, then
\begin{equation}
J^*-J(\widehat\pi)
\leq\sup_{\pi\in\PiSet}\epsint(\pi)+\delta_p+\epsint(\widehat\pi).
\label{eq:planning-bound}
\end{equation}
\end{corollary}
Appendix~\ref{app:value} gives the Bellman telescope and full regret derivation;
Appendix~\ref{app:transport} treats shifts beyond verified occupancy.

\begin{corollary}[Predictive accuracy does not certify a decision]\label{cor:absolute}
Arbitrarily small uniform transition error can reverse world-model-greedy action
rankings; total variation one can coexist with exact values for every policy.
\end{corollary}
Appendix~\ref{app:separation} gives both constructions: qualification must compare
prediction error with the advantage at stake.

\subsection{From a certified comparison to closed-loop planning}

Replanning changes a certified continuation. Fix a reference policy $\pi_0$, with
$Q_t^0=r+PV_{t+1}^{\pi_0}$ and
$\widehat Q_t^0=\widehat r+\widehat P\widehat V_{t+1}^{\pi_0}$.
For candidate action kernel $\kappa_t$, let $|Q_t^0-\widehat Q_t^0|\leq b_t$ and define
\begin{align}
d_t(s)&=\int Q_t^0(s,a)(\kappa_t-\pi_{0,t})(\mathrm da\mid s),
\label{eq:local-advantage}\\
\widehat d_t(s)&=\int\widehat Q_t^0(s,a)(\kappa_t-\pi_{0,t})(\mathrm da\mid s),
\label{eq:imagined-local-advantage}\\
B_t(s)&=\int b_t(s,a)(\kappa_t+\pi_{0,t})(\mathrm da\mid s).
\label{eq:local-radius}
\end{align}

\begin{theorem}[Closed-loop reliability]\label{thm:closed-loop}
Suppose simultaneously at all admissible states that $\widehat d_t\geq S_t-\xi_t$,
where $\xi_t\geq0$. Put
\begin{equation}
\begin{aligned}
T_t&=S_t-B_t-\xi_t,
&g_t&=\ind\{T_t>0\},\\
\nu_t&=g_t\kappa_t+(1-g_t)\pi_{0,t}.
\end{aligned}
\label{eq:closed-gate}
\end{equation}
If $\rho_t^\nu$ is the true state law under $\nu$, then
\begin{equation}
\begin{aligned}
J(\nu)-J(\pi_0)
&=\sum_{t=0}^{H-1}\int_{\mathcal S}\rho_t^\nu(\mathrm ds)
  \int_{\mathcal A}Q_t^0(s,a)[\nu_t-\pi_{0,t}](\mathrm da\mid s)\\
&=\sum_{t=0}^{H-1}\int_{\mathcal S}g_t(s)d_t(s)\rho_t^\nu(\mathrm ds)\\
&\geq\sum_{t=0}^{H-1}\int_{\mathcal S}g_t(s)
 [\widehat d_t(s)-B_t(s)]\rho_t^\nu(\mathrm ds)\\
&\geq\sum_{t=0}^{H-1}\int_{\mathcal S}g_t(s)T_t(s)
 \rho_t^\nu(\mathrm ds)\geq0.
\end{aligned}
\label{eq:closed-performance}
\end{equation}
Strict improvement holds if a positive-margin intervention is reached with positive
probability.
\end{theorem}

Appendix~\ref{app:closed-loop} proves the result under repeated replanning. The
guarantee remains decision-specific: it compares local predicted advantage against
continuation-value uncertainty under the state distribution induced by the gated
policy. Under $|r-\widehat r|\leq\epsilon_r$ and
$\TV(P,\widehat P)\leq\epsilon_p$, Theorem~\ref{thm:planning} gives
\begin{equation}
b_t=(H-t)\epsilon_r+\tfrac12(H-t)(H-t-1)R_b\epsilon_p.
\label{eq:conditional-radius}
\end{equation}
Appendix~\ref{app:closed-loop} gives tighter and truncated-rollout variants;
Appendix~\ref{app:gradient} gives the corresponding differential certificates.

\section{Dual-Frontier Learning}
\label{sec:dual}

\subsection{Qualification and promotion}

Let $S(x)$ estimate $\widehat\Gamma(x)=\Jh(\pi_1)-\Jh(\pi_0)$, $B(x)$ bound
world-model distortion, and $\xi(x)$ bound estimation error. Define
\begin{align}
T(x)&=S(x)-B(x)-\xi(x),
\label{eq:decision-margin}\\
\FW&=\{x:S(x)>\xi(x),\ B(x)\geq S(x)-\xi(x)\},
\label{eq:frontiers}\\
\FA&=\{x:T(x)>0\}.
\label{eq:agent-frontier}
\end{align}
$\FW$ retains promising but uncertified comparisons for verification, whereas
$\FA$ contains certified improvements; $S\leq\xi$ is deferred.

\paragraph{Calibrated uncertainty.}
For $e(x)=|\widehat\Gamma(x)-\Gamma_\M(x)|$, fit a nonnegative score $u$
independently of $n$ exchangeable calibration requests and set
\begin{equation}
q_j=e_j-u(x_j),
\qquad k=\lceil(n+1)(1-\alpha)\rceil,
\qquad B(x)=[u(x)+q_{(k)}]_+.
\label{eq:main-calibration}
\end{equation}
Proposition~\ref{prop:conformal} gives $\Prob\{e(x)>B(x)\}\leq\alpha$. If
$\widehat\Gamma\geq S-\xi$ except with probability $\delta$, then
\begin{equation}
\begin{aligned}
\{T(x)>0,\Gamma_\M(x)\leq0\}
&\subseteq
\{e(x)>B(x)\}\cup\{\widehat\Gamma(x)<S(x)-\xi(x)\},\\
\Prob\{T(x)>0,\Gamma_\M(x)\leq0\}
&\leq\alpha+\delta.
\end{aligned}
\label{eq:main-marginal}
\end{equation}
Appendix~\ref{app:calibration} gives the rank proof, noisy-label extension, and
selection-conditional distinction.

\subsection{Reusable evidence and adaptive planning}

For a stationary finite task with $D$ states, $m$ state--action rows, shared known
rewards, and $n$ samples per row, let $\overline P$ be empirical and define
\begin{equation}
\begin{gathered}
a_n=\sqrt{\frac{D\log2+\log(4m/\alpha)}{2n}},
\qquad
\mathcal C=\left\{\mathcal N:
\max_{s,a}\TV(P_\mathcal N,\overline P)\leq a_n\right\}.
\end{gathered}
\label{eq:main-audit}
\end{equation}
Appendix~\ref{app:fresh} gives $\Prob\{\M\in\mathcal C\}\geq1-\alpha$. Put
$K=R_bH(H-1)$ and, for selected $\widehat P_i$,
\begin{equation}
u_i=\max_{s,a}\min\{1,\TV(\overline P,\widehat P_i)+a_n\}.
\label{eq:uniform-request-radius}
\end{equation}

\begin{theorem}[Adaptive reuse of a world certificate]\label{thm:shared}
Let $x_i,\widehat P_i$ depend on the audit and all previous observations. Conditional
on their selection, assume $\Prob(F_i^c\mid\mathcal H_i)\leq\delta_i$, where
$F_i=\{\widehat\Gamma_i\geq S_i-\xi_i\}$. Accept only
\begin{equation}
T_i=S_i-Ku_i-\xi_i>0.
\label{eq:dual-margins}
\end{equation}
For
$E=\{\M\in\mathcal C\}$ and
$\mathcal B=\bigcup_{i\geq1}\{T_i>0,\Gamma_\M(x_i)\leq0\}$,
\begin{equation}
\begin{aligned}
\mathcal B&\subseteq E^c\cup\bigcup_{i\geq1}F_i^c,\\
\Prob(\mathcal B)
&\leq\Prob(E^c)+\sum_{i\geq1}
\E[\Prob(F_i^c\mid\mathcal H_i)]
\leq\alpha+\sum_{i\geq1}\delta_i.
\end{aligned}
\label{eq:shared-risk}
\end{equation}
The audit costs $mn$ queries, independently of subsequent comparisons.
\end{theorem}
Appendix~\ref{app:shared} gives the full adaptive proof and extensions.

\subsection{Progress and evidence cost}

For successive accepted behaviors with certified margins $T_i>0$, simultaneous
validity gives
\begin{align}
J(\pi_n)-J(\pi_0)
&\geq\sum_{i=0}^{n-1}T_i,
\label{eq:decision-progress}\\
\left|\{i\in\{0,\ldots,n-1\}:T_i\geq\tau\}\right|
&\leq\frac{HR_b-J(\pi_0)}{\tau},
\qquad \tau>0.
\label{eq:large-promotions}
\end{align}
Thus a bounded objective cannot sustain a fixed positive certified margin indefinitely
(Appendix~\ref{app:stationarity}).

With shared rewards and $H\geq2$, let
$q_i=\max_{s,a}\TV(P,\widehat P_i)$ and
$a_i=\widehat\Gamma_i-Kq_i$. For exact world-model evaluation, on $E$,
\begin{align}
T_i&\geq a_i-2Ka_n,
\label{eq:margin-audit-price}\\
n&>\frac{2}{s^2}[D\log2+\log(4m/\alpha)]
\label{eq:amortized-size}
\end{align}
certifies every request with $a_i/K\geq s$ simultaneously
(Appendix~\ref{app:amortized}).

Conversely, distinguishing true advantages $\pm c\gamma$,
$\gamma\in(0,1/4]$, with promotion probabilities at least $1-\delta$ and at most
$\delta$ requires
\begin{equation}
n\geq\frac{\operatorname{kl}(1-\delta,\delta)}{16\gamma^2},
\qquad 0<\delta<\tfrac12.
\label{eq:verification-lower}
\end{equation}
Appendices~\ref{app:evidence} and~\ref{app:sequential} give the lower bound and
adaptive accounting. Together, the bounds establish the inverse-square statistical
price of reliable promotion with reusable evidence.

\clearpage
\noindent\begin{minipage}{\textwidth}
\refstepcounter{table}\label{tab:agentworld-main}
\raggedright Table~\thetable: Agent--world-model evaluation on BFCL v4, API-Bank, and NexusRaven. We report Task Success (Success) and Parameter Accuracy (Acc.); Avg. is the uniform mean across benchmarks.\par
\vspace{2pt}
\centering
\scriptsize
\setlength{\tabcolsep}{3.25pt}
\renewcommand{\arraystretch}{0.90}
\resizebox{0.90\textwidth}{!}{%
\begin{tabular}{@{}cc*{8}{c}@{}}
\toprule
\addlinespace[2.5pt]
\multirow{2}{*}{Models} & \multirow{2}{*}{Methods}
& \multicolumn{2}{c}{BFCL v4} & \multicolumn{2}{c}{API-Bank}
& \multicolumn{2}{c}{NexusRaven} & \multicolumn{2}{c}{Avg.}\\
\cmidrule(lr){3-4}\cmidrule(lr){5-6}\cmidrule(lr){7-8}\cmidrule(l){9-10}
& & Success & Acc. & Success & Acc. & Success & Acc. & Success & Acc.\\
\midrule
\multirow{6}{*}{\shortstack[c]{Llama-3.1-8B\\Instruct}}
& Agent-only    & 69.12 & 85.31 & 64.96 & 78.76 & 56.29 & 72.36 & 63.46 & 78.81\\
& Always-WM     & 22.50 & 60.79 & 66.73 & 79.20 & 63.21 & 80.33 & 50.81 & 73.44\\
& Confidence    & 35.12 & 69.51 & 88.58 & 92.55 & 82.39 & 87.89 & 68.70 & 83.32\\
& Consistency   & 46.38 & 74.52 & 97.24 & 98.00 & 96.54 & 96.83 & 80.05 & 89.78\\
& Pessimistic   & 41.62 & 72.14 & 94.29 & 95.67 & 90.88 & 92.91 & 75.60 & 86.91\\
& \textbf{Dual-Frontier} & \textbf{78.12} & \textbf{90.30} & \textbf{98.62} & \textbf{98.83} & \textbf{98.11} & \textbf{98.01} & \textbf{91.62} & \textbf{95.71}\\
\midrule
\multirow{6}{*}{Qwen3-8B}
& Agent-only    & 76.38 & 90.78 & 71.26 & 80.00 & 74.21 & 81.31 & 73.95 & 84.03\\
& Always-WM     & 21.88 & 60.38 & 67.52 & 80.68 & 62.58 & 80.68 & 50.66 & 73.91\\
& Confidence    & 35.00 & 69.41 & 91.34 & 94.06 & 91.19 & 93.02 & 72.51 & 85.50\\
& Consistency   & 47.75 & 75.42 & 98.03 & 98.10 & 97.11 & 97.80 & 80.96 & 90.44\\
& Pessimistic   & 41.62 & 73.18 & 93.11 & 95.24 & 94.97 & 96.27 & 76.57 & 88.23\\
& \textbf{Dual-Frontier} & \textbf{79.38} & \textbf{92.20} & \textbf{98.92} & \textbf{98.64} & \textbf{97.80} & \textbf{97.96} & \textbf{92.03} & \textbf{96.27}\\
\midrule
\multirow{3}{*}{\shortstack[c]{Frontier\\Models}}
& GPT-6 Astra      & 82.19 & 89.56 & 84.24 & 85.14 & 87.40 & 91.35 & 84.61 & 88.68\\
& Claude Opus 5    & 84.47 & 93.96 & 64.72 & 93.35 & 67.62 & 96.85 & 72.27 & 94.72\\
& GLM-5.3-Flash    & 85.78 & 80.86 & 67.00 & 79.08 & 79.53 & 86.46 & 77.44 & 82.13\\
\bottomrule
\end{tabular}
}
\end{minipage}
\vspace{8pt}
\noindent\begin{minipage}{\textwidth}
{\centering
\includegraphics[width=\linewidth]{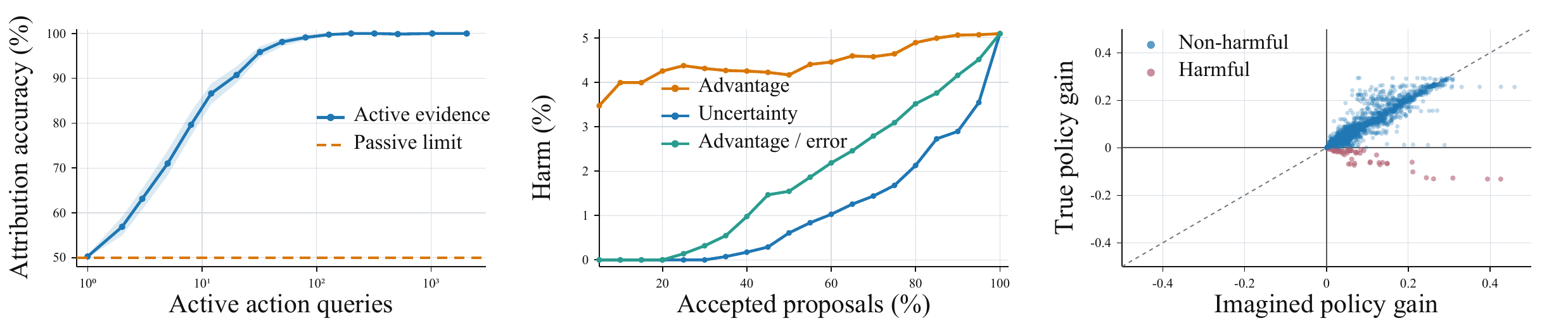}\par}
\refstepcounter{figure}\label{fig:requests}
\raggedright Figure~\thefigure: Controlled finite-world validation of intervention attribution, planning qualification, and imagined-update transfer.\par
\end{minipage}
\vspace{-8pt}
\section{Experiments}
\label{sec:experiments}
\enlargethispage{18pt}
Our experiments progressively test whether the theory identifies when a learned world model can be used reliably and whether Dual-Frontier converts that diagnosis into better decisions. We proceed in two stages: controlled finite-world experiments directly isolate and validate attribution, qualification, and evidence allocation under auditable conditions, after which a cross-backbone evaluation on public agent benchmarks directly instantiates the same verify-then-promote rule in modern tool-use pipelines across heterogeneous model-task settings.

\subsection{Controlled validation with learned world models}
\label{sec:controlled-validation}

\begin{wrapfigure}[11]{r}{0.48\textwidth}
\vspace{-0.9em}
\centering
\includegraphics[width=0.492\linewidth]{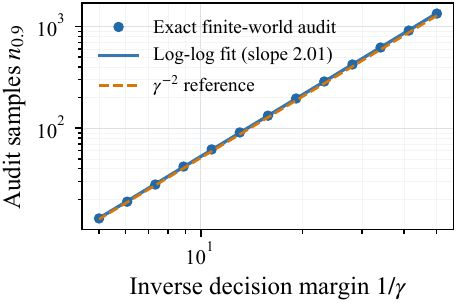}%
\hspace{0.004\linewidth}%
\includegraphics[width=0.492\linewidth]{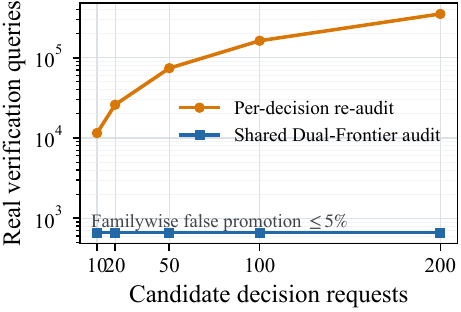}
\refstepcounter{figure}\label{fig:scaling-reuse}
\par\vspace{2pt}
\raggedright Figure~\thefigure: Evidence complexity and audit reuse. Left: inverse-square certification cost. Right: shared-certificate verification.
\vspace{-0.2em}
\end{wrapfigure}
We train action-conditioned learned world models in sparse, chain, and grid finite-horizon worlds and test held-out decisions. Each environment family contains 16 states and four actions, with horizons $H\in\{4,8,12\}$, providing controlled variation in both dynamics and planning depth. We use 40 development worlds to fix the protocol, 80 independent worlds for finite-world auditing, and 240 disjoint held-out worlds for final evaluation, with separate identifiers and random streams across all three roles. The learned world model is a Beta-smoothed empirical transition kernel fitted from sampled row observations, while a separately acquired 32-query-per-row model determines the fixed reference behavior. Given a reference action and candidate intervention, the model predicts the counterfactual consequence and the router decides whether to alter the action; exact returns are evaluation-only. The evaluation combines 400 paired constructions that share the learned model and passive record but differ on one unobserved action, held-out planning and teaching requests comparing unconditional positive-imagination acceptance, uncertainty-only rejection, and decision-relative qualification, and online runs contrasting ungated, uniformly gated, and decision-directed evidence acquisition under the same real-transition budget.
\begin{table}[h]
\caption{Dual-Frontier component ablation with Llama-3.1-8B-Instruct on three benchmarks. We report Task Success (Success), Parameter Accuracy (Acc.), and uniform benchmark averages.}
\label{tab:agentworld-ablation}
\vspace{4pt}
\centering
\scriptsize
\setlength{\tabcolsep}{3.25pt}
\renewcommand{\arraystretch}{0.96}
\resizebox{\textwidth}{!}{%
\begin{tabular}{@{}cc*{8}{c}@{}}
\toprule
\addlinespace[2.5pt]
\multirow{2}{*}{Models} & \multirow{2}{*}{Methods}
& \multicolumn{2}{c}{BFCL v4} & \multicolumn{2}{c}{API-Bank}
& \multicolumn{2}{c}{NexusRaven} & \multicolumn{2}{c}{Avg.}\\
\cmidrule(lr){3-4}\cmidrule(lr){5-6}\cmidrule(lr){7-8}\cmidrule(l){9-10}
& & Success & Acc. & Success & Acc. & Success & Acc. & Success & Acc.\\
\midrule
\multirow{4}{*}{\shortstack[c]{Llama-3.1-8B\\Instruct}}
& Advantage-only             & 65.41 & 82.60 & 94.60 & 96.12 & 93.71 & 95.40 & 84.57 & 91.37\\
& w/o world-model error      & 69.82 & 85.57 & 96.49 & 97.40 & 95.61 & 96.61 & 87.31 & 93.19\\
& w/o estimation error       & 76.41 & 89.33 & 98.02 & 98.23 & 97.39 & 97.70 & 90.61 & 95.09\\
& \textbf{Dual-Frontier}     & \textbf{78.12} & \textbf{90.30} & \textbf{98.62} & \textbf{98.83} & \textbf{98.11} & \textbf{98.01} & \textbf{91.62} & \textbf{95.71}\\
\bottomrule
\end{tabular}
}
\end{table}
\begin{figure}[h]
\centering
\includegraphics[width=\linewidth]{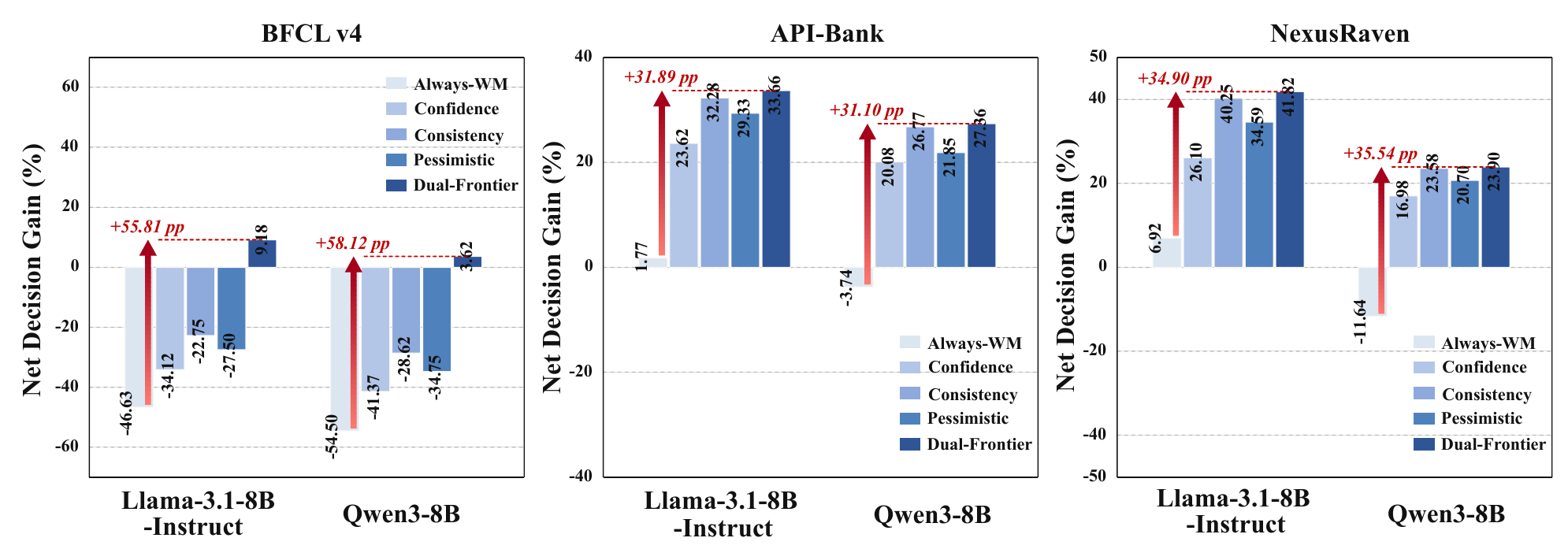}
\caption{Net decision gain across benchmarks and backbones. Red arrows show Dual-Frontier's NDG improvement over Always-WM.}
\label{fig:fig3}
\end{figure}
All methods share the world model, reference behavior, proposals, calibration split, and evidence budget, varying only structure, horizon, margin, and evidence. Appendix~\ref{app:exact-exp} gives the construction and protocol. Figure~\ref{fig:requests} shows that targeted evidence raises attribution accuracy from $50.25\%$ to $100\%$ and cuts effect error fivefold ($0.0900$ to $0.0189$); over 11,520 requests, qualification reduces harmful use from $5.10\%$ to $0.30\%$. Certified imagined updates preserve beneficial transfer while rejecting harmful revisions. Figure~\ref{fig:scaling-reuse} yields a log--log exponent of $2.014$ when varying only the true decision margin, matching $n=\Theta(\gamma^{-2})$; a shared certificate handles 200 adaptive requests with 664 versus 350,400 real queries ($527.7\times$ fewer) while keeping familywise false-promotion risk below $5\%$. These results validate counterfactual attribution, decision-relative qualification, inverse-margin scaling, and reusable certification. Appendix~\ref{app:paired-exp} reports full curves, ablations, and per-environment results.
\subsection{Agent-world-model evaluation across benchmarks}
\label{sec:agent-evaluation}
\begin{wraptable}[10]{r}{0.45\textwidth}
\vspace{-0.8em}
\centering
\caption{Robustness of Dual-Frontier across three independent world-model generation seeds. Results are mean $\pm$ standard deviation over a fixed randomly sampled benchmark subset.}
\label{tab:seed-robustness}
\vspace{3pt}
\scriptsize
\setlength{\tabcolsep}{3.4pt}
\renewcommand{\arraystretch}{1.00}
\begin{tabular}{lcc}
\toprule
Backbone & Success  & Acc.  \\
\midrule
Llama-3.1-8B-Instruct & $91.48 \pm 0.63$ & $95.57 \pm 0.41$ \\
Qwen3-8B              & $91.91 \pm 0.57$ & $96.18 \pm 0.35$ \\
\bottomrule
\end{tabular}
\vspace{-0.4em}
\end{wraptable}
We next instantiate the same Dual-Frontier decision rule in realistic agent--world-model interaction, using Qwen-AgentWorld as the learned world model~\citep{zuo2026qwenagentworld}. We evaluate Llama-3.1-8B-Instruct~\citep{grattafiori2024llama3} and Qwen3-8B~\citep{yang2025qwen3}. The benchmarks separately cover function calling in BFCL v4~\citep{patil2025bfcl}, API use in API-Bank~\citep{li2023apibank}, and heterogeneous tool schemas in NexusRaven~\citep{nexusraven2023}. For each request, the agent first emits its base action, after which Qwen-AgentWorld proposes candidate revisions together with predicted consequences. To isolate the trust decision, all routing rules evaluate the same fixed candidate and share identical prompts, schemas, candidate records, and inference budgets; they differ only in whether and how that proposal is admitted. This matched protocol separates selective qualification from proposal quality or generation cost. We also test frontier models~\citep{openai2026gpt6astra,anthropic2026claudeopus5,zai2026glm53flash}. Appendix~\ref{app:benchmark-rules} details the calibrated routing rule, and Appendix~\ref{app:benchmark-metrics} defines the metrics.

\clearpage
\begin{figure}[!ht]
    \centering
    \includegraphics[width=\linewidth]{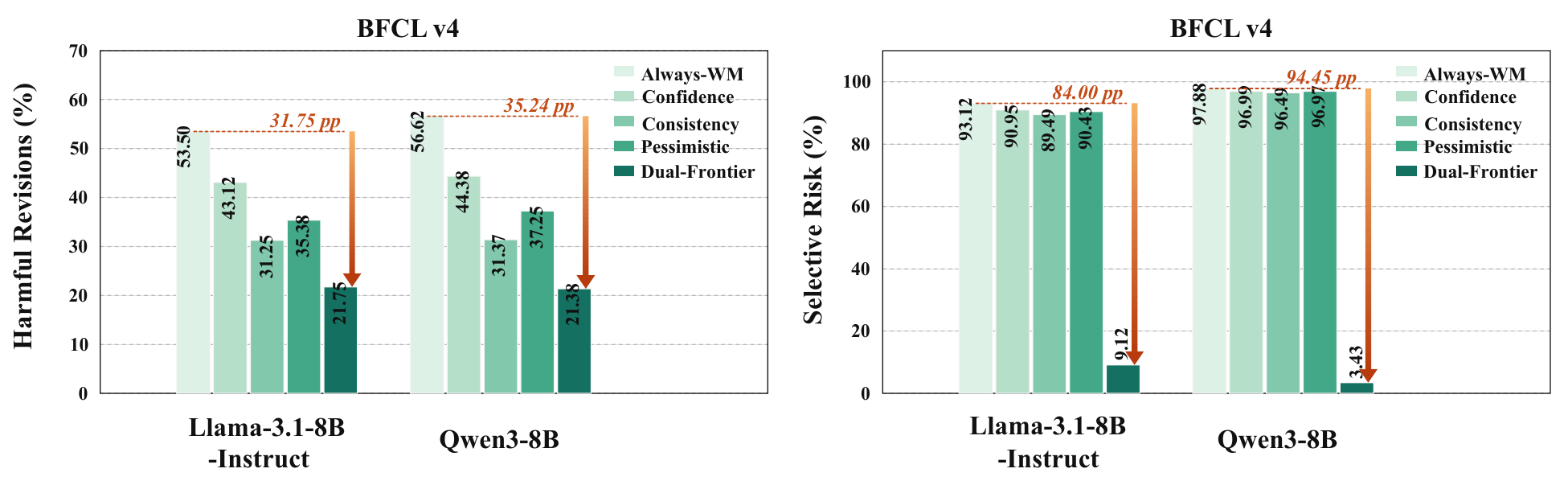}
    \caption{BFCL v4 reliability. Dual-Frontier reduces harmful revisions and selective risk across both agent backbones.}
    \label{fig:bfcl-reliability}
    \vspace{-8pt}
\end{figure}
\begin{figure}[!ht]
\centering
\includegraphics[width=\linewidth]{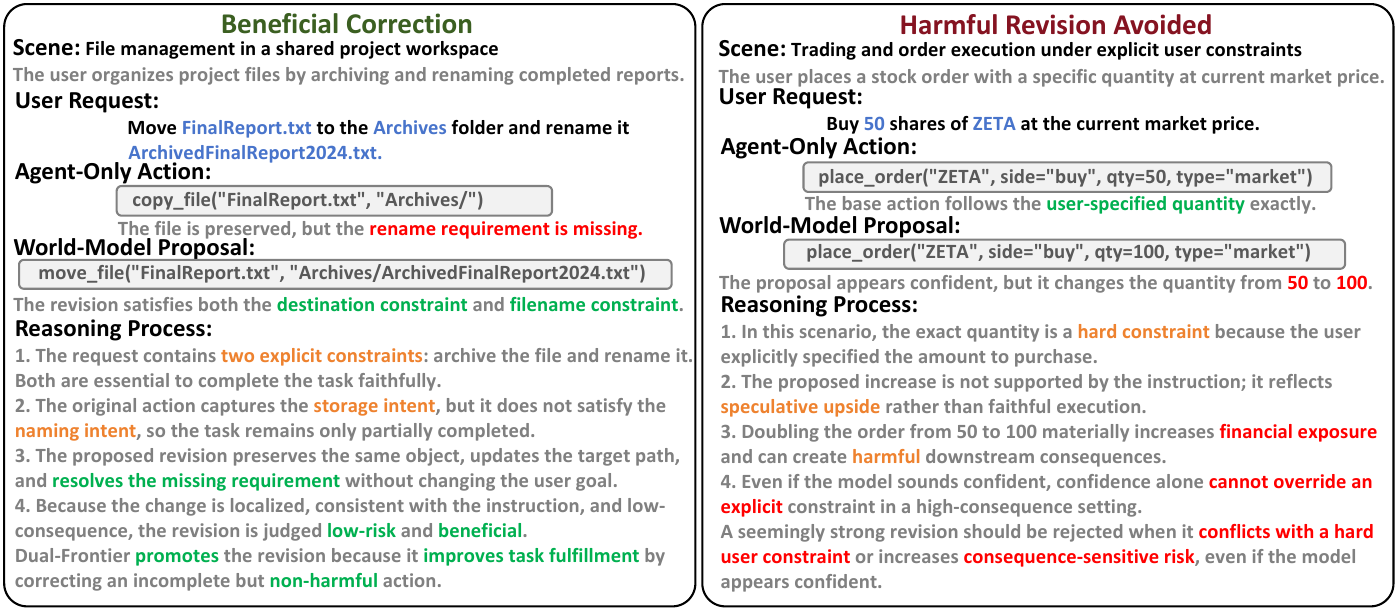}
\caption{Representative BFCL v4 cases. Dual-Frontier promotes a reliable correction while deferring a confident but high-risk revision.}
\label{fig:case-study}
\vspace{-8pt}
\end{figure}
\FloatBarrier

Table~\ref{tab:agentworld-main} shows that Dual-Frontier leads every benchmark--metric pair for both backbones. Against the strongest competing rule, average Success improves by $11.57$/$11.07$ points and Acc. by $5.93$/$5.83$ for Llama-3.1-8B-Instruct/Qwen3-8B. Table~\ref{tab:agentworld-ablation} shows monotonic recovery as the two uncertainty coordinates are restored; Appendix~\ref{app:benchmark-additional} gives the Qwen3-8B ablation. Figure~\ref{fig:fig3} reports positive gains on all benchmarks, exceeding Always-WM by $31.10$--$58.12$ points, while Table~\ref{tab:seed-robustness} shows limited variation across three generation seeds.

Figure~\ref{fig:bfcl-reliability} shows lower harmful revisions and selective risk on BFCL v4 for both backbones, and Figure~\ref{fig:case-study} illustrates the same rule on individual requests. Together with the coverage--risk results in Appendix~\ref{app:benchmark-additional}, these findings attribute the gains to selectively admitted revisions rather than more aggressive world-model use.

Dual-Frontier preserves the base action when estimated benefit is not sufficiently separated from model and finite-generation uncertainty. This request-level policy matters because the same proposal can help one request and harm another even when predicted benefits appear similar. It therefore avoids treating the world model as globally reliable or unreliable. Because prompts, schemas, candidate records, and inference budgets are fixed, the reliability change reflects qualification rather than a stronger proposal generator. Across both backbones, lower harmful revisions accompany higher task success and parameter accuracy; on API-Bank and NexusRaven, the rule retains high coverage while rejecting revisions that fail the certified margin. The ablations reinforce the same mechanism---benefit proposes an intervention, but evidence determines whether it is promoted.

Agreement between task success and parameter accuracy rules out a simple trade between tool selection and malformed arguments. Dual-Frontier improves both, indicating more correct complete actions rather than locally plausible revisions. Consistency across datasets, backbones, ablations, and seeds supports an evidence-sensitive trust mechanism.

\clearpage

\subsection{In-depth Analysis}
\label{sec:in-depth-analysis}
\begin{wrapfigure}[14]{r}{0.43\textwidth}
\vspace{-0.85em}
\centering
\includegraphics[width=\linewidth]{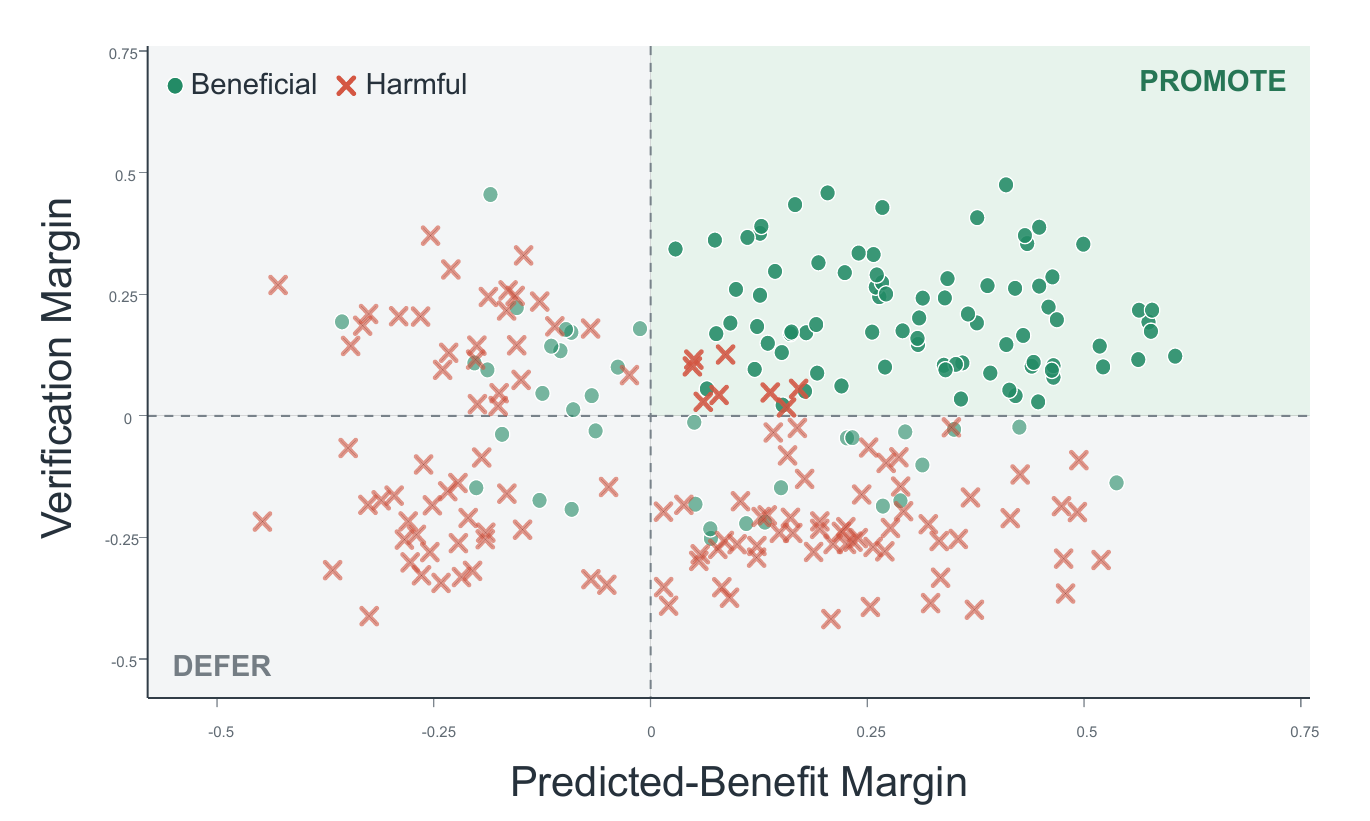}
\caption{Frontier geometry over 240 sampled requests.}
\label{fig:frontier-mechanism}
\vspace{-0.15em}
\end{wrapfigure}

We finally examine the mechanism behind Dual-Frontier.
Figure~\ref{fig:frontier-mechanism} shows that beneficial revisions cluster in the
positive-benefit, positive-verification region, whereas many harmful proposals remain
outside the promotion frontier despite favorable predicted benefit. This indicates that
predicted benefit alone is insufficient and verification is needed to screen unreliable
interventions. Figure~\ref{fig:case-study} provides complementary BFCL v4 cases, where
Dual-Frontier promotes a supported correction but defers a confident, high-risk revision.
Together, these observations align with Sections~\ref{sec:setup}--\ref{sec:dual}:
decision-level qualification separates promising predictions from those sufficiently
supported for action. Controlled finite-world experiments and public agent benchmarks
consistently validate the same frontier mechanism under formal certification and realistic tool use, confirming its robustness across both settings.
\FloatBarrier

\vspace{-0.4em}
\section{Conclusion}
\vspace{-0.20em}
Dual-Frontier establishes a decision-specific view of world-model reliability. We show that passive failures cannot identify whether an error originates from the agent or the world model, motivating explicit qualification against model and estimation uncertainty. This yields a verify-then-promote principle with decision-level guarantees, closed-loop improvement, and reusable verification under adaptive evidence allocation. Controlled finite-world experiments and agent benchmarks consistently validate the resulting attribution and qualification behavior. More broadly, reliable world-model use should be judged at the decision boundary: the key question is not whether a model is accurate in general, but whether current evidence is sufficient to justify the action that depends on it.
\clearpage

\bibliography{references}
\bibliographystyle{iclr2027_conference}

\clearpage

\appendix

\section{Mathematical Interface and Notation}
\label{app:interface}

For reference, the paired world notation is
\begin{equation*}
\M_z=(\mathcal S,\mathcal A,P_z,r_z,\rho_z,H),
\qquad
\Mh_z=(\mathcal S,\mathcal A,\widehat P_z,\widehat r_z,\rho_z,H).
\end{equation*}

\subsection{Objects held fixed in an intervention}

The task index $z$ identifies a reference environment, not merely a generated description. Once $z$ is fixed, both worlds share state space $\mathcal S$, action space $\mathcal A$, initial law $\rho$, and horizon $H$. Their transition kernels are $P,\widehat P$ and bounded rewards are $r,\widehat r$. All are measurable. An episode contains states and actions at times $0,\ldots,H-1$; the transition following the final reward is omitted. This convention explains why transition error is summed only through $H-2$.

The kernel construction on standard Borel spaces defines a unique trajectory law. For a measurable policy $\pi$, let $p^\pi$ and $\widehat p^\pi$ denote these two laws; $\mu_t^\pi$ and $\widehat\mu_t^\pi$ are their state--action marginals. Policies may depend on time, already included in the state. $V_t^\pi(s)$ and $\widehat V_t^\pi(s)$ denote expected rewards from time $t$ onward. Returns are finite and lie in $[0,HR_b]$. Suprema over policy classes are used when maximizers are not guaranteed to exist.

$\mathsf A_\phi$ is a measurable, possibly randomized model-to-policy operator. Returns of randomized outputs average the return of each realized policy over a fixed seed law. This is not generally the return of the pointwise mean of Markov policy kernels. Oracle replacement changes the model argument only: the operator, available actions, budget, and seed law remain fixed. An oracle that also changes the optimizer or observation interface defines a different intervention.

In the optional differential analysis, $\theta\in\Theta\subseteq\mathbb R^d$ parameterizes a behavior evaluated by predictive policy search. It is distinct from the fixed competence description $\phi$ used to define counterfactual attribution. At an update, the kernels and rewards do not depend on the differentiation variable. Models may be refitted between updates; every certificate must then be checked for the new snapshot.

\subsection{Partial observability and learned representations}

An observable history can serve as the state: at time $t$ take the full observation--action history, including any observed rewards or tool outputs. The true kernel is the conditional law of the next observable history under an action, and the learned kernel predicts the same object. This is a common measurable interface on which the finite-horizon arguments apply. It does not assume access to hidden physical states or equality of internal representations.

For a latent model supporting a history-dependent policy, one route is to specify a history-space kernel, for example through a decoder. Alternatively, when the policy consumes only $f(h)$, Proposition~\ref{prop:latent} compares true history-conditional latent transitions directly with the learned latent kernel. Its uniform error includes representation aliasing; neither a decoder nor exact Markov sufficiency is then assumed. Passive latent prediction loss alone supplies neither guarantee. Likewise, a generated executable environment that defines a new task is not automatically an approximation to a reference environment. The evaluation correspondence must be specified.

\subsection{Predictive decision procedures}
\label{app:uses}
A planner with specification $\phi$ maps the supplied world $\mathcal N$ to a behavior $\mathsf A_\phi(\mathcal N)$. It may select a plan once or replan after each observation. With a fixed internal supplied world, the latter procedure defines a history-based policy. Randomized search, finite imagined rollouts and a fixed rollout budget are part of the operator; no parameter update is needed.

Attribution recomputes the operator with the oracle world and compares $J_\M(\mathsf A_\phi(\M))$ with $J_\M(\mathsf A_\phi(\Mh))$. Certification instead first constructs $\pi_0,\pi_1$ and evaluates these same policies in every plausible world. In particular, evaluating $J_\mathcal N(\pi_1)$ does not replace the internal model or rerun the candidate search. A realized randomized output can be conditioned on before fresh evaluation; reuse of the data that selected it requires a simultaneous certificate.

For repeated action selection, a policy-pair certificate applies to the specified complete continuation. Theorem~\ref{thm:closed-loop} instead fixes the reference continuation in each local action comparison and controls the resulting policy through the performance-difference identity. This distinction connects a frozen predictive model to an interactive agent without assuming that a short imagined plan will actually be followed.

\subsection{Notation ledger}

\begin{table}[ht]
\centering
\caption{Notation for attribution and use-dependent reliability.}
\label{tab:notation}
\small
\begin{tabular}{>{\raggedright\arraybackslash}p{0.22\textwidth}>{\raggedright\arraybackslash}p{0.69\textwidth}}
\toprule
Symbol & Meaning\\
\midrule
$z,t,i,k$ & Task, within-episode time, sample or accepted-update index, verification round\\
$\M,\Mh$ & True environment and its learned approximation on a common interface\\
$P,\widehat P,r,\widehat r$ & Transition kernels and reward functions\\
$\rho,H,R_b$ & Shared initial law, horizon, and common nonnegative reward bound\\
$\PiSet,\mathsf A_\phi$ & Reference policy class and fixed model-to-policy operator\\
$J,\Jh,J^*$ & True return, model return, supremal true return in $\PiSet$\\
$\widehat\pi,\pi^\circ$ & Deployed policy and oracle-model counterfactual\\
$R,A,W$ & Observed regret, intrinsic agent regret, signed model effect\\
$\mu_t^\pi,\widehat\mu_t^\pi$ & True and learned state--action occupancies\\
$\Delta_t^\pi,\epsint(\pi)$ & Cross-model Bellman residual and value-error certificate\\
$\ell_t,\delta$ & Policy-independent error envelope and policy-TV radius; $\delta$ denotes a sampling error probability where stated\\
$\psi_\theta,G,S_t$ & Policy score, uniform score bound, accumulated score through $t$\\
$d_t,e,D_t$ & Expected local transition error, its sum, and prefix coupling bound\\
$g,\widehat g,\widetilde g$ & True gradient, exact imagined gradient, sampled imagined gradient\\
$C_H,b_\theta,\xi$ & Trajectory-score bound, model gradient-bias bound, sampling radius\\
$b,h,L,\eta$ & Total gradient radius, sampled gradient norm, smoothness, step size\\
$x,u,q_i,U_\alpha,T$ & Snapshot, uncertainty score, residual, calibrated upper bound, promotion margin\\
$\mathcal H_k,\varepsilon_k,N_K$ & Verification history, round risk budget, count of non-improving executed steps\\
$x,\pi_0,\pi_1$ & Decision request, reference policy, candidate policy\\
$\Gamma_{\M},\mathcal C,\ell_{\mathcal C}$ & True policy contrast, world confidence set, robust contrast\\
$\FW,\FA$ & Model-evidence frontier and certified agent frontier\\
$S,B,\xi,T$ & Estimated contrast, model-error radius, estimation radius, certified margin (return units)\\
\bottomrule
\end{tabular}
\end{table}

Local symbols are defined when first introduced. $\|\cdot\|_2$ is Euclidean norm, $\langle\cdot,\cdot\rangle$ its inner product, and the norm of a matrix is its induced operator norm. $[a]_+=\max\{a,0\}$. $\E$ and $\Prob$ denote expectation and probability under the law specified in context.

\subsection{Logical dependencies}

\begin{table}[ht]
\centering
\caption{What each guarantee requires and what it does not supply.}
\label{tab:assumption-ledger}
\small
\begin{tabular}{>{\raggedright\arraybackslash}p{0.23\textwidth}>{\raggedright\arraybackslash}p{0.33\textwidth}>{\raggedright\arraybackslash}p{0.33\textwidth}}
\toprule
Result & Required information or regularity & Not implied\\
\midrule
Theorem~\ref{thm:impossibility} & Passive episodes of a fixed composition & Impossibility after interventions\\
Theorem~\ref{thm:planning} & Common measurable interface; bounded rewards & An error bound from arbitrary pixel loss\\
Proposition~\ref{prop:trust-transport} & Uniform policy-TV control & Reliability after unrestricted policy search\\
Theorem~\ref{thm:closed-loop} & Simultaneous conditional value bounds; fixed reference & Guaranteed improvement from an arbitrary critic\\
Theorem~\ref{thm:gradient} & Frozen worlds; common reward; bounded policy score & Gradient fidelity for shared-parameter model updates\\
Theorem~\ref{cor:safe-update} & Gradient-error ball; local smoothness; feasible step & Harmfulness of every rejected step\\
Proposition~\ref{prop:conformal} & Independent score fitting; exchangeable true error labels & Coverage conditional on selection\\
Proposition~\ref{prop:fresh} & Finite spaces; fresh generative row queries & Cheap verification in arbitrary visual domains\\
Theorem~\ref{thm:shared} & One stationary task; uniform world confidence; valid estimation & Acceptance of all future proposals\\
Theorem~\ref{thm:sequential} & Conditional risk control after adaptive task choice & Improvement of every task when only one task is audited\\
Appendix~\ref{app:stationarity} & Fixed objective; relative accuracy; accepted updates & A bound on waiting time or real samples\\
Theorem~\ref{thm:allocation} & Explicit no-transfer, monotone precedence response & Universal advantage over scalar curricula\\
\bottomrule
\end{tabular}
\end{table}

\FloatBarrier
\section{Attribution: Impossibility, Instability, and Recovery}
\label{app:attribution}

The attribution continuum in Theorem~\ref{thm:impossibility} is
\begin{equation*}
(A,W)=(\lambda,c-\lambda).
\end{equation*}

\subsection{Complete fixed-operator construction}

\paragraph{Proof of Theorem~\ref{thm:impossibility}.}

Let $\mathcal S=\{s,g,b\}$ and $\mathcal A=\{a_0,a_1,a_2\}$. Start at $s$, use horizon two, and set $r(s,a)=r(b,a)=0$ and $r(g,a)=c$ for every action. The states $g,b$ are absorbing. For $\lambda\in[0,c]$, define
\begin{equation*}
\begin{aligned}
P_\lambda(g\mid s,a_0)&=0,&
P_\lambda(g\mid s,a_1)&=1-\lambda/c,&
P_\lambda(g\mid s,a_2)&=1.
\end{aligned}
\end{equation*}
The remaining probability goes to $b$. The supplied world model $\widehat P$ sends every action at $s$ to $b$ and agrees with the true kernels elsewhere. The rewards are known and shared.

The operator evaluates only actions $a_0,a_1$ using its supplied world model and chooses the higher-return action, breaking ties toward $a_0$. Its terminal action is fixed arbitrarily. This same operator is used for every $\lambda$; it is not hidden from the observer. Its restricted search is a concrete competence limitation, while the reference class contains the policy choosing $a_2$.

Under $\widehat P$, deployment always chooses $a_0$. Under $P_\lambda$, this produces $(s,a_0,0,b,a_0,0)$ with probability one. For $\lambda<c$, oracle replacement chooses $a_1$ and achieves $c-\lambda$; at $\lambda=c$ the tie-breaking rule chooses $a_0$ and still achieves $c-\lambda=0$. The reference value is always $c$. Hence $R=c,A=\lambda,W=c-\lambda$.

Let $\mathcal O$ be generated by any number of deployed episodes, the supplied world model, the known operator, and independent randomization used by an estimator. All these objects have the same law for every $\lambda$. For any $\mathcal O$-measurable real estimate $X$,
\begin{align*}
\sup_{\lambda\in[0,c]}\E_\lambda|X-\lambda|
&\geq \max\{\E_0|X|,\E_c|X-c|\}\\
&\geq\frac12\bigl(\E_0|X|+\E_c|X-c|\bigr)\\
&=\frac12\int\bigl(|x|+|x-c|\bigr)\,\mathsf Q(\mathrm dx)\\
&\geq\frac c2,
\end{align*}
where $\mathsf Q$ is the common law of $X$ and the last step is the triangle inequality. Conversely, $X=c/2$ obeys
\[
\sup_{\lambda\in[0,c]}|X-\lambda|=c/2,
\]
so the lower bound is exact. The endpoint classification argument follows from the same common-law experiment: the two endpoint labels have equal prior probability and identical observations, hence Bayes and minimax error are both $1/2$. This proves Theorem~\ref{thm:impossibility} for arbitrary passive sample size, including infinite passive repetition.

Notice that the supplied world model is exact on the deployed occupancy: its error there is zero. The missing information is about untried actions. Direct observation of the world-model parameters does not remove this obstruction. Repeated controlled trials of $a_1$ identify its success probability asymptotically; a single trial need not determine the attribution. The construction isolates why passive return and on-policy fit cannot alone justify an attribution or a certificate for new interventions.

\subsection{Finite-sample Monte Carlo planning}
\label{app:training-attribution}

The obstruction persists when action values are estimated from finitely many imagined rollouts. Use the same states, reward, supplied world model, and reference class as above, but let $P(g\mid s,a_1)=p$ range over $[0,1]$. Fix a positive integer $N$. The planner samples $N$ independent supplied-world episodes starting with each of $a_0,a_1$, compares the sample return means, and outputs the corresponding deterministic initial action. Ties select $a_0$; terminal actions are fixed.

Under the supplied world model all $2N$ imagined episodes return zero, so the complete search transcript and output policy are identical for all true worlds. Subsequent real deployment of $a_0$ is also identical. Under oracle replacement, $a_0$ still always returns zero, whereas $a_1$ is selected if and only if at least one of its $N$ imagined episodes succeeds. Its selection probability is $1-(1-p)^N$. Evaluation uses a new episode independent of search, giving
\begin{equation*}
J(\pi^\circ)=c f_N(p),\qquad
f_N(p)=p[1-(1-p)^N].
\end{equation*}
The function $f_N$ is continuous, has $f_N(0)=0,f_N(1)=1$, and for $p\in(0,1)$,
\[
f_N'(p)=1-(1-p)^N+Np(1-p)^{N-1}>0.
\]
For every $\lambda\in[0,c]$, there is therefore a unique $p$ with $c f_N(p)=c-\lambda$. Since $J^*=c$ and $J(\widehat\pi)=0$, the same $(A,W)=(\lambda,c-\lambda)$ continuum results. The common-law estimation and classification proofs apply even when the observer also sees every imagined search record.

The planner, rollout budget, and supplied world model are fixed throughout this family. The oracle does not increase compute or give a better optimizer; it changes only the predictive law used for look-ahead. This finite-sample construction requires no change of agent parameters.

\subsection{A fixed true-world variant}

A different construction keeps the true one-step MDP fixed but varies the agent operator. Use actions with rewards $c,0$ and one supplied world model distinct from the truth. For each $\lambda$, let the operator choose the zero-reward action under the supplied world model and the good action with probability $1-\lambda/c$ under the true model. This gives the same decomposition and minimax bounds already in a one-state MDP. Theorem~\ref{thm:impossibility} is stronger with respect to what the observer knows about the operator: there the operator and supplied world model are fixed.

\subsection{Predictive reliability does not identify a black-box operator's model effect}

\begin{proposition}[No universal modulus for failure ownership]\label{prop:operator}
There exist one fixed agent operator and true one-step world such that, for every sufficiently small $\epsilon>0$, a supplied world model satisfies
$\sup_{\pi\in\PiSet}|J(\pi)-\Jh(\pi)|\leq\epsilon$
but $W=c$ for a constant $c>0$ independent of $\epsilon$.
\end{proposition}

\begin{proof}
Use three actions with true rewards $0,0,c$. The model changes only the first reward from zero to $\epsilon$, with $0<\epsilon\leq R_b$. Define the operator to choose the second action if the supplied first reward is positive, and the third action otherwise. This is one fixed measurable operator. It achieves zero in deployment and $c$ after oracle replacement, so $W=c$. The value of any policy changes by at most $\epsilon$. Thus no function $\omega(\epsilon)\to0$ can universally bound $|W|$ for arbitrary operators.
\end{proof}

A continuity or optimization-residual assumption on the operator can exclude this example, but it cannot be omitted. The value and update comparisons in the main text avoid that extra assumption by certifying the policies actually compared. Their validity is not a claim that they estimate $A$ or $W$.

The sign of $W$ is equally unrestricted. Reverse the two branches of the operator in the preceding example. The wrong model now induces the good action and oracle replacement the bad action, giving $W=-c$. The decomposition remains an exact signed identity.

\subsection{Finite-sample attribution with an oracle intervention}

\begin{proposition}[Interventional recovery]\label{prop:attribution-sampling}
Fix $\widehat\pi,\pi^\circ$ and a reference policy $\pi_r$ before evaluation, with
$0\leq J^*-J(\pi_r)\leq\omega$.
Use $m$ independent episodes of each policy and write their sample means as $\overline J,\overline J^\circ,\overline J_r$. For $\delta\in(0,1)$, put
\begin{equation*}
t=HR_b\sqrt{\frac{\log(6/\delta)}{2m}},
\qquad
\widehat A=\overline J_r-\overline J^\circ,\qquad
\widehat W=\overline J^\circ-\overline J.
\end{equation*}
With probability at least $1-\delta$,
\begin{equation*}
|\widehat A-A|\leq2t+\omega,\qquad
|\widehat W-W|\leq2t.
\end{equation*}
\end{proposition}

\begin{proof}
A bounded return has range $HR_b$. Hoeffding's inequality makes each sample mean accurate to $t$ except with probability $\delta/3$. A union bound yields simultaneous accuracy. Each difference has sampling error at most $2t$; only the agent-regret estimate has the additional reference-policy error $\omega$. Cross-policy independence is unnecessary, so paired random numbers are allowed if each policy's episodes remain independent.
\end{proof}

This result requires the oracle-model counterfactual policy, not merely more episodes of deployment. It quantifies evaluation after the identifying intervention; it does not remove the intervention's cost.

\section{Decision Reliability: Full Derivations}
\label{app:value}

Theorem~\ref{thm:planning} also controls planning regret. If
$\Jh(\widehat\pi)\geq\sup_{\pi\in\PiSet}\Jh(\pi)-\delta_p$, then the
planning-regret inequality in Corollary~\ref{cor:planning-regret} follows.
The proof below takes a supremum and does not require an optimal policy to exist.

\subsection{Robust policy contrasts and the order of evaluation}
\label{app:contrast}

Fix a use request $x$, including the two output policies, and a nonempty class $\mathcal C$ of worlds sharing the evaluation interface. Bounded rewards imply
$-HR_b\leq\Gamma_{\mathcal N}(x)\leq HR_b$. Let $\underline\ell(x)$ be any computable lower bound on $\ell_{\mathcal C}(x)$. By the definition of infimum,
\begin{equation*}
\M\in\mathcal C
\quad\Longrightarrow\quad
\Gamma_{\M}(x)\geq\ell_{\mathcal C}(x)
\geq\underline\ell(x).
\end{equation*}
On an event where $\M\in\mathcal C$, this deterministic implication holds for all requests at once, including any selected as a function of $\mathcal C$. No union bound over policies or use requests is required. For random sets and selectors we assume the displayed infima and events are measurable; the explicit finite-state bounds used in the paper are measurable functions of finite arrays.

A uniform improvement guarantee with positive constant $a$ exists if and only if $\ell_{\mathcal C}(x)>0$: one direction takes infima, and the other chooses $a=\ell_{\mathcal C}(x)$. This equivalence is for the information represented by $\mathcal C$, not for the unknown true world alone. A computable lower bound may fail even when the exact infimum is positive.

If $\mathcal C$ is compact and $\Gamma_{\mathcal N}(x)$ is continuous in $\mathcal N$, its image is a nonempty compact subset of $\mathbb R$ and contains its infimum. Thus $\ell_{\mathcal C}(x)\leq0$ exhibits a plausible world where the comparison is non-improving. Without attainment, $\ell_{\mathcal C}=0$ may coexist with strictly positive contrast in every world, as for contrasts $1/j$, $j\geq1$. It still precludes a uniform positive margin.

For finite spaces and finite horizon, the expectation of a fixed history-based policy is a finite sum of products of transition probabilities and rewards. It is continuous in these arrays. Closed row-confidence sets in the probability simplices are compact and nonempty; the empirical kernel is feasible. These observations justify attainment for the concrete verifier. Exact contrast minimization can remain computationally difficult, which is why the main text supplies tractable conservative bounds instead of assuming access to an exact robust optimizer.

A request compares a proposed behavior with an explicit reference behavior. The certificate does not establish global optimality; the planning bound controls suboptimality separately when an optimization residual is available.

\subsection{Bellman operators and the telescoping identity}

\paragraph{Proof of Theorem~\ref{thm:planning} and Corollary~\ref{cor:planning-regret}.}

For a policy $\pi$, define $r^\pi(s)=\int r(s,a)\pi(\mathrm da\mid s)$ and
\begin{equation*}
(P^\pi f)(s)=\int_{\mathcal A}\int_{\mathcal S}
 f(s')P(\mathrm ds'\mid s,a)\pi(\mathrm da\mid s).
\end{equation*}
Time dependence is understood through the augmented state. Both Bellman recursions hold for bounded measurable values:
$V_t^\pi=r^\pi+P^\pi V_{t+1}^\pi$ and
$\widehat V_t^\pi=\widehat r^\pi+\widehat P^\pi\widehat V_{t+1}^\pi$.
Subtracting gives the exact operator identity
\begin{align*}
V_t^\pi-\widehat V_t^\pi
&=r^\pi-\widehat r^\pi
  +P^\pi V_{t+1}^\pi-\widehat P^\pi\widehat V_{t+1}^\pi\\
&=(r-\widehat r)^\pi
  +(P^\pi-\widehat P^\pi)\widehat V_{t+1}^\pi
  +P^\pi(V_{t+1}^\pi-\widehat V_{t+1}^\pi).
\end{align*}
Let $\rho_t$ be the true state law, $d_t=V_t^\pi-\widehat V_t^\pi$, and $b_t=\int\Delta_t^\pi(\cdot,a)\pi(\mathrm da\mid\cdot)$. These local symbols are used only in this proof. Since $\rho_{t+1}=\rho_tP^\pi$ and $d_H=0$, the operator recursion above yields
\begin{align*}
\int d_t\,\mathrm d\rho_t
&=\int b_t\,\mathrm d\rho_t
  +\int P^\pi d_{t+1}\,\mathrm d\rho_t\\
&=\E_{\mu_t^\pi}\Delta_t^\pi
  +\int d_{t+1}\,\mathrm d\rho_{t+1},\\
\int d_0\,\mathrm d\rho_0
&=\sum_{t=0}^{H-1}\E_{\mu_t^\pi}\Delta_t^\pi
  +\int d_H\,\mathrm d\rho_H\\
&=\sum_{t=0}^{H-1}\E_{\mu_t^\pi}\Delta_t^\pi.
\end{align*}
Because the initial law is shared, the left-hand side is exactly $J(\pi)-\Jh(\pi)$.

For finite signed measure $P-Q$ of total mass zero and bounded $f$, let $a=(\sup f+\inf f)/2$. Then
\[
|(P-Q)f|=|(P-Q)(f-a)|
 \leq 2\TV(P,Q)\|f-a\|_\infty
 =\spn(f)\TV(P,Q).
\]
Applying this inequality to the residual gives
\begin{align*}
|J(\pi)-\Jh(\pi)|
&=\left|\sum_{t=0}^{H-1}\int\Delta_t^\pi\,\mathrm d\mu_t^\pi\right|\\
&\leq\sum_{t=0}^{H-1}\int
 \left(|r-\widehat r|+
 \left|\int\widehat V_{t+1}^\pi\,\mathrm d(P-\widehat P)\right|\right)
 \mathrm d\mu_t^\pi\\
&\leq\sum_{t=0}^{H-1}\int
 \left(|r-\widehat r|+\spn(\widehat V_{t+1}^\pi)\TV(P,\widehat P)\right)
 \mathrm d\mu_t^\pi\\
&=\epsint(\pi).
\end{align*}
No common support of the transition kernels is required.

For the two-policy comparison, let $e_\pi=J(\pi)-\Jh(\pi)$. Then
\begin{align*}
J(\pi')-J(\pi)
&=\Jh(\pi')-\Jh(\pi)+e_{\pi'}-e_\pi\\
&\geq\Jh(\pi')-\Jh(\pi)-|e_{\pi'}|-|e_\pi|\\
&\geq\Jh(\pi')-\Jh(\pi)-\epsint(\pi')-\epsint(\pi).
\end{align*}
The symmetric argument gives the corresponding upper bound. If an optimal $\pi^*$ exists, the planning bound sharpens to
\begin{equation*}
J(\pi^*)-J(\widehat\pi)
\leq\delta_p+\epsint(\pi^*)+\epsint(\widehat\pi).
\end{equation*}
If not, apply this inequality to a sequence approaching $J^*$, retaining the uniform supremum used in the main theorem.

\subsection{Uniform bounds, equality examples, and horizon scaling}
\label{app:sharpness}

If $|r-\widehat r|\leq\epsilon_r$ and $\TV(P,\widehat P)\leq\epsilon_p$ uniformly, then $0\leq\widehat V_{t+1}^\pi\leq(H-t-1)R_b$ implies
\begin{equation*}
|J(\pi)-\Jh(\pi)|
\leq H\epsilon_r+\frac{H(H-1)}2R_b\epsilon_p.
\end{equation*}
Both model rewards and true rewards must obey the stated reward bound.

\begin{proposition}[Two endpoints are necessary]\label{prop:planning-sharp}
For every $\beta\in(0,R_b/2]$, the right-hand side of the endpoint inequality above is attained with $\delta_p=0$.
\end{proposition}

\begin{proof}
Take a one-step two-action task with true rewards $2\beta,0$ and model rewards $\beta,\beta$. Let the world-model optimizer break the tie toward the second action. Each endpoint value error is $\beta$, and the true regret is $2\beta$. Thus one endpoint error cannot simply be removed from a general planning comparison.
\end{proof}

The transition coefficient in the uniform value bound above is first-order sharp. Consider a deterministic true chain with zero initial reward and reward $R_b$ at each later nonabsorbing state. The model enters a zero-reward absorbing state with independent probability $\epsilon_p$ at each transition. The actual discrepancy is
\begin{equation*}
R_b\sum_{t=1}^{H-1}\left[1-(1-\epsilon_p)^t\right]
=\frac{H(H-1)}2R_b\epsilon_p+O(H^3\epsilon_p^2)
\end{equation*}
as $\epsilon_p\downarrow0$ for fixed $H$. A uniform reward offset attains the linear reward coefficient. These examples establish value-bound scaling, not minimax sharpness of the gradient horizon exponent.

\subsection{Transport to a new policy}
\label{app:transport}

Define the policy-independent envelope
$\ell_t=|r-\widehat r|+(H-t-1)R_b\TV(P,\widehat P)$. We prove the transport inequality the transport inequality above.

\begin{proposition}[Verified-policy transport]\label{prop:trust-transport}
If $\sup_s\TV(\pi(\cdot\mid s),\pi_0(\cdot\mid s))\leq\delta$, then
\begin{equation*}
\epsint(\pi)\leq
\sum_{t=0}^{H-1}\E_{\mu_t^{\pi_0}}\ell_t+
R_b\sum_{t=0}^{H-1}(H-t)\min\{1,(t+1)\delta\}.
\end{equation*}
\end{proposition}

\begin{proof}
Couple identical initial states. Whenever states agree, use a maximal coupling of the action kernels; conditional action disagreement has probability at most $\delta$. If actions also agree, sample the identical true transition synchronously. After disagreement, any coupling preserving both marginal dynamics suffices.

The event that state--action pairs disagree by time $t$ requires at least one of $t+1$ action disagreements. A union bound, or a first-disagreement decomposition, gives
\[
\TV(\mu_t^\pi,\mu_t^{\pi_0})\leq\min\{1,(t+1)\delta\}.
\]
For $0\leq f\leq B$, the signed-measure argument above gives
$\E_\mu f-\E_\nu f\leq B\TV(\mu,\nu)$.
Apply it to $\ell_t$, whose range lies in $[0,(H-t)R_b]$. Since the integrand defining $\epsint(\pi)$ is bounded by $\ell_t$, summing proves the transport inequality above. Finally,
\[
\sum_{t=0}^{H-1}(H-t)(t+1)=\frac{H(H+1)(H+2)}6.
\]
\end{proof}

A coverage assumption provides another route. If
$\mu_t^\pi\ll\mu_t^{\pi_0}$ with Radon--Nikodym derivative bounded by $c_t$, then nonnegativity gives
\begin{equation*}
\epsint(\pi)\leq\sum_{t=0}^{H-1}c_t\E_{\mu_t^{\pi_0}}\ell_t.
\end{equation*}
This is a stated density-ratio assumption, not something guaranteed by low passive prediction loss. If a candidate policy visits an action absent from the validation support, neither this bound nor unrestricted reuse of an on-policy certificate is justified.

\subsection{Wasserstein and discounted variants}

Suppose $(\mathcal S,d)$ is Polish, both next-state kernels have finite first moments, and $\widehat V_{t+1}^\pi$ is $L_{t+1}$-Lipschitz. The Kantorovich--Rubinstein dual formula gives
\begin{equation*}
|J(\pi)-\Jh(\pi)|
\leq\sum_{t=0}^{H-1}\E_{\mu_t^\pi}
 \left[|r-\widehat r|+L_{t+1}W_1(P,\widehat P)\right].
\end{equation*}
Here $W_1$ is the first Wasserstein distance for metric $d$. The proof substitutes
$|(P-\widehat P)\widehat V_{t+1}^\pi|\leq L_{t+1}W_1(P,\widehat P)$
into the exact identity. The additional Lipschitz assumption is indispensable; a small state-space metric error does not control an arbitrary discontinuous continuation value.

For a stationary discounted problem with $\gamma\in(0,1)$, define
$J^\gamma(\pi)=\E\sum_{t\geq0}\gamma^t r_t$ and
$\mu_\gamma^\pi=(1-\gamma)\sum_{t\geq0}\gamma^t\mu_t^\pi$.
The bounded Bellman resolvent yields
\begin{equation*}
J^\gamma(\pi)-\Jh^\gamma(\pi)
=\frac1{1-\gamma}\E_{\mu_\gamma^\pi}
 [r-\widehat r+\gamma(P-\widehat P)\widehat V^\pi].
\end{equation*}
To see this, iterate
$V-\widehat V=(r-\widehat r)^\pi+
\gamma(P^\pi-\widehat P^\pi)\widehat V+\gamma P^\pi(V-\widehat V)$.
The remaining term after $n$ iterations has sup norm at most
$2\gamma^n R_b/(1-\gamma)$ and vanishes. Hence uniform reward and transition bounds imply
\begin{equation*}
|J^\gamma(\pi)-\Jh^\gamma(\pi)|
\leq\frac{\epsilon_r}{1-\gamma}
 +\frac{\gamma R_b\epsilon_p}{(1-\gamma)^2}.
\end{equation*}
This is a value extension only; the main finite-horizon statistical constants are not reused unchanged in infinite horizon.

\subsection{Conditional certificates and closed-loop composition}
\label{app:closed-loop}
\paragraph{Proof of Theorem~\ref{thm:closed-loop}.}
Fix a reference continuation $\pi_0$. For every admissible history state $s$ at time $t$ and first action $a$, let $\pi^a$ choose $a$ first and follow $\pi_0$ thereafter. This is a policy on the residual horizon $H-t$ with initial law concentrated at $s$. Applying the residual identity to this conditional problem gives
\begin{equation*}
Q_t^0(s,a)-\widehat Q_t^0(s,a)
=\sum_{j=t}^{H-1}\E_{P,\pi^a}\!\left[\Delta_j^{\pi^a}(s_j,a_j)\mid s_t=s\right].
\end{equation*}
Here $\Delta_j^{\pi^a}$ uses the learned continuation value of that same policy. Thus a valid conditional envelope is
\begin{equation*}
b_t(s,a)=\sum_{j=t}^{H-1}\E_{P,\pi^a}
 \left[|r-\widehat r|+\spn(\widehat V_{j+1}^{\pi_0})\TV(P,\widehat P)\mid s_t=s\right].
\end{equation*}
It is an analytical quantity unless its components are certified. Uniform row errors give
\begin{equation*}
b_t(s,a)\leq
\sum_{j=t}^{H-1}[\epsilon_r+(H-j-1)R_b\epsilon_p]
=(H-t)\epsilon_r+\frac{(H-t)(H-t-1)}2R_b\epsilon_p.
\end{equation*}
This proves~\eqref{eq:conditional-radius}. Finite-state audits bound all these conditional problems at once, including histories that a deployed agent has not yet visited.

\paragraph{Contrast-sensitive refinement.}
Let $\lambda_t=\kappa_t(\cdot\mid s)-\pi_{0,t}(\cdot\mid s)$, a signed measure of mass zero, and let $|\lambda_t|$ denote its total-variation measure. Then
\begin{equation*}
|d_t-\widehat d_t|
=\left|\int(Q_t^0-\widehat Q_t^0)\,\mathrm d\lambda_t\right|
\leq\int b_t\,\mathrm d|\lambda_t|
\leq\int b_t\,\mathrm d(\kappa_t+\pi_{0,t}).
\end{equation*}
The first upper bound is the exact support function of the pointwise uncertainty class:
\begin{equation*}
\sup_{|f(a)|\leq b_t(s,a)}
\left|\int f\,\mathrm d\lambda_t\right|
=\int b_t\,\mathrm d|\lambda_t|.
\end{equation*}
To prove equality, take $f=b_t\,\mathrm d\lambda_t/\mathrm d|\lambda_t|$ on the support of $|\lambda_t|$. The Radon--Nikodym derivative is $+1$ or $-1$ almost everywhere, so this choice is admissible and realizes the integral. Sharpness is for the stated value-error class, not a claim that every extremizer is a realizable MDP. For uniform $b_t=b$ the refined radius is $2b\,\TV(\kappa_t,\pi_{0,t})$, and it vanishes when the proposed action law is unchanged. Either radius may be used in Theorem~\ref{thm:closed-loop}.

\paragraph{The performance-difference identity.}
Let $r_t=r(s_t,a_t)$ and write $\E_\nu$ for expectation under the true trajectory law of the gated policy. By conditional expectation and the reference Bellman identity,
\begin{align*}
\E_\nu[r_t+V_{t+1}^{\pi_0}(s_{t+1})\mid s_t]
 &=\int Q_t^0(s_t,a)\nu_t(\mathrm da\mid s_t),
\\
V_t^{\pi_0}(s_t)
 &=\int Q_t^0(s_t,a)\pi_{0,t}(\mathrm da\mid s_t),
\\
\E_\nu[r_t+V_{t+1}^{\pi_0}(s_{t+1})-V_t^{\pi_0}(s_t)]
 &=\E_{\rho_t^\nu}[g_t d_t].
\end{align*}
Summing the left side cancels every intermediate value:
\begin{align*}
J(\nu)-J(\pi_0)
&=\sum_{t=0}^{H-1}\E_\nu
 [r_t+V_{t+1}^{\pi_0}(s_{t+1})-V_t^{\pi_0}(s_t)]\\
&=\sum_{t=0}^{H-1}\int_{\mathcal S}g_t(s)d_t(s)\,\rho_t^\nu(\mathrm ds)\\
&\geq\sum_{t=0}^{H-1}\int_{\mathcal S}g_t(s)
 [\widehat d_t(s)-B_t(s)]\,\rho_t^\nu(\mathrm ds)\\
&\geq\sum_{t=0}^{H-1}\int_{\mathcal S}g_t(s)T_t(s)\,\rho_t^\nu(\mathrm ds)\geq0.
\end{align*}
The first equality uses $V_H^{\pi_0}=0$ and $\E_\rho V_0^{\pi_0}=J(\pi_0)$; the inequalities use $|d_t-\widehat d_t|\leq B_t$, $\widehat d_t\geq S_t-\xi_t$, and $g_t=\ind\{T_t>0\}$. A finite sum of nonnegative integrable terms is strictly positive if one term is positive on an event of positive probability. The result guarantees expected improvement, not samplewise dominance of realized rewards.

\paragraph{Short rollouts and terminal estimates.}
Suppose a conditional model rollout stops after $h\in\{1,\ldots,H-t\}$ steps and bootstraps with a measurable function $\widetilde V$ satisfying
$\|\widetilde V-\widehat V_{t+h}^{\pi_0}\|_\infty\leq\omega$.
Let $\widetilde Q_t^0$ be the resulting expected truncated return. The tower property in the learned world gives
\begin{equation*}
|\widetilde Q_t^0-\widehat Q_t^0|\leq\omega,\qquad
|\widetilde Q_t^0-Q_t^0|\leq b_t+\omega.
\end{equation*}
Use $b_t+\omega$ in~\eqref{eq:local-radius} and add the separate Monte Carlo estimation radius. Without a terminal-value error bound, an arbitrary score for one predicted observation is not a continuation-value certificate. The same applies to truncated language-model reasoning scores and learned critics.

\section{Differential Reliability of Predictive Policy Search}
\label{app:gradient}

The parameter $\theta$ indexes behaviors evaluated through one fixed learned world model. The following results are optional local sufficient conditions for the same policy contrast used in the main text. They require differentiable policies and are not assumed for discrete action selection or frozen language-agent inference.

\subsection{Differential certificates for predictive policy search}
Assume shared rewards and freeze both worlds while differentiating. Policy densities on an open $\Theta\subseteq\mathbb R^d$ have common support, are continuously differentiable, and have score
$\psi_\theta=\nabla_\theta\log\pi_\theta$ with $\|\psi_\theta\|_2\leq G$.
Let $d_t=\E_{\mu_t^{\pi_\theta}}\TV(P,\widehat P)$,
$D_t=\min\{1,\sum_{j<t}d_j\}$, $e=\sum_{t=0}^{H-2}d_t$, and
$C_H=GR_bH(H+1)/2$. Empty sums are zero.
\begin{theorem}[Gradient reliability]\label{thm:gradient}
For $g=\nabla_\theta J(\pi_\theta)$ and $\widehat g=\nabla_\theta\Jh(\pi_\theta)$,
\begin{equation*}
\|g-\widehat g\|_2\leq
2GR_b\sum_{t=0}^{H-1}(t+1)D_t=:b_\theta
\leq2C_H\min\{1,e\}.
\end{equation*}
If $b_\theta<\|\widehat g\|_2$, then
$\langle g,\widehat g\rangle\geq\|\widehat g\|_2(\|\widehat g\|_2-b_\theta)>0$.
\end{theorem}
\paragraph{Proof of Theorem~\ref{thm:gradient}.}
Let $Q_t,\widehat Q_t$ denote the true and learned laws of the prefix through $(s_t,a_t)$, and set $S_t=\sum_{j=0}^t\psi_\theta(s_j,a_j)$. Differentiation under the trajectory integral and score centering give
\begin{equation*}
g=\sum_{t=0}^{H-1}\E_P[S_t r(s_t,a_t)],
\qquad
\widehat g=\sum_{t=0}^{H-1}\E_{\widehat P}[S_t r(s_t,a_t)].
\end{equation*}
Sequential coupling gives $\TV(Q_t,\widehat Q_t)\leq D_t$. Thus
\begin{equation*}
\|g-\widehat g\|_2\leq
\sum_{t=0}^{H-1}2\|S_t r(s_t,a_t)\|_\infty\TV(Q_t,\widehat Q_t)
\leq b_\theta,
\end{equation*}
where $\|F\|_\infty=\sup\|F\|_2$ for a vector-valued function. Cauchy--Schwarz and the gradient-ball geometry yield, for $b_\theta<\|\widehat g\|_2$,
\begin{equation*}
\langle g,\widehat g\rangle
 \geq\|\widehat g\|_2(\|\widehat g\|_2-b_\theta)>0,
\qquad
\cos(g,\widehat g)\geq
\sqrt{1-\frac{b_\theta^2}{\|\widehat g\|_2^2}}.
\end{equation*}
Here $\cos(g,\widehat g)=\langle g,\widehat g\rangle/(\|g\|_2\|\widehat g\|_2)$.
The angular bound is attained in dimension at least two within the gradient-error ball. The following subsections supply the domination, coupling, and extremal calculations; this is geometric sharpness, not MDP minimax optimality.

\begin{theorem}[Local improvement certificate]\label{cor:safe-update}
Suppose $\|\widetilde g-\widehat g\|_2\leq\xi$ and $J$ is $L$-smooth along a feasible step $\theta^+=\theta+\eta\widetilde g$, with $L,\eta>0$. Set $b=b_\theta+\xi$ and $h=\|\widetilde g\|_2$. Then
\begin{equation*}
J(\pi_{\theta^+})-J(\pi_\theta)
\geq\eta h[(1-L\eta/2)h-b].
\end{equation*}
If $h>b$, the feasible choice $\eta=(h-b)/(Lh)$ guarantees $(h-b)^2/(2L)$ improvement. If $h\leq b$, no direction has positive worst-case first-order gain over the gradient ball.
\end{theorem}
\paragraph{Proof of Theorem~\ref{cor:safe-update}.}
By the triangle inequality, $\|g-\widetilde g\|_2\leq b$. Smoothness therefore gives
\begin{equation*}
J(\pi_{\theta^+})-J(\pi_\theta)
\geq\eta\langle g,\widetilde g\rangle-\tfrac L2\eta^2h^2
\geq\eta h(h-b)-\tfrac L2\eta^2h^2.
\end{equation*}
For a feasible displacement $v$, the exact supporting lower model is
\begin{equation*}
\inf_{\|q-\widetilde g\|_2\leq b}
 \left\{\langle q,v\rangle-\frac L2\|v\|_2^2\right\}
=\langle\widetilde g,v\rangle-b\|v\|_2-\frac L2\|v\|_2^2.
\end{equation*}
For $v\ne0$, equality holds at $q=\widetilde g-bv/\|v\|_2$. At length $t=\|v\|_2$, the maximum is $(h-b)t-Lt^2/2$, attained by alignment with $\widetilde g$ when $h>0$. Its maximizing length is $[h-b]_+/L$, where $[x]_+=\max\{x,0\}$. If $h\leq b$, zero belongs to the gradient ball. Quadratic objectives attain the lower model; sharpness is relative to this local uncertainty class.
Appendix~\ref{app:robust} proves sharpness and the general-displacement certificate; Appendix~\ref{app:implementation} states the requirements for practical optimizers.

We record the regularity and prefix constants used in this section:
\begin{equation*}
\psi_\theta(s,a)=\nabla_\theta\log\pi_\theta(a\mid s),
\qquad
\|\psi_\theta(s,a)\|_2\leq G.
\end{equation*}
\begin{equation*}
\begin{aligned}
d_t&=\E_{\mu_t^{\pi_\theta}}\TV(P,\widehat P),
&e&=\sum_{t=0}^{H-2}d_t,\\
D_t&=\min\!\left\{1,\sum_{j<t}d_j\right\},
&C_H&=\frac{GR_bH(H+1)}2,
\end{aligned}
\end{equation*}
With $S_t=\sum_{j=0}^t\psi_\theta(s_j,a_j)$, the common trajectory representation is given by the common trajectory representation above.

\subsection{Differentiation under the trajectory integral}

For each state, use a common sigma-finite action reference measure and a strictly positive policy density on parameter-independent support. Fix $\theta$ and a compact ball around it contained in $\Theta$, with the bound the uniform score condition above throughout that ball. The mean-value theorem gives
\begin{equation*}
\frac{\pi_{\theta+v}(a\mid s)}{\pi_\theta(a\mid s)}
\leq\exp(G\|v\|_2).
\end{equation*}
Let $Q_\theta$ be the trajectory law in one fixed world and $Z_v=\mathrm dQ_{\theta+v}/\mathrm dQ_\theta$. For $\|v\|_2\leq a$ within the chosen ball,
\begin{align*}
Z_v&=\prod_{j=0}^{H-1}\frac{\pi_{\theta+v}(a_j\mid s_j)}{\pi_\theta(a_j\mid s_j)}
\leq e^{HGa},\\
\|\nabla_v Z_v\|_2&=\left\|Z_v\sum_{j=0}^{H-1}\psi_{\theta+v}(s_j,a_j)\right\|_2
\leq HG e^{HGa}.
\end{align*}
The return is bounded by $HR_b$. Dominated convergence permits differentiation under $Q_\theta$. This argument is applied separately in each world and requires no likelihood ratio between $P$ and $\widehat P$.

Differentiating policy normalization gives
$\int\psi_\theta(s,a)\pi_\theta(\mathrm da\mid s)=0$.
Let $r_t=r(s_t,a_t)$. The likelihood-ratio identity initially reads
\[
\nabla J=\E_P\left[\left(\sum_{j=0}^{H-1}\psi_\theta(s_j,a_j)\right)
                         \left(\sum_{t=0}^{H-1}r_t\right)\right].
\]
Let $\mathcal H_j=\sigma(s_0,a_0,\ldots,a_{j-1},s_j)$ be the history before action $a_j$. For $j>t$, $r_t$ is $\mathcal H_j$-measurable, and
\begin{equation*}
\E[\psi_\theta(s_j,a_j)r_t]
=\E\!\left[r_t\E[\psi_\theta(s_j,a_j)\mid\mathcal H_j]\right]=0.
\end{equation*}
Removing these terms gives
\begin{equation*}
\nabla J=\E_P F_\theta,\qquad
F_\theta=\sum_{t=0}^{H-1}S_t r_t,\qquad
S_t=\sum_{j=0}^t\psi_\theta(s_j,a_j).
\end{equation*}
Since $\|S_t\|_2\leq(t+1)G$, $\|F_\theta\|_2\leq C_H$. The same functional represents the imagined gradient when rewards are shared.

\subsection{Sequential coupling without a support assumption}

For standard Borel kernels a measurable maximal coupling can be constructed from their common part. At a state--action pair let $\nu=P+\widehat P$, write their densities relative to $\nu$ as $p,\widehat p$, and use $\min(p,\widehat p)\nu$ as the common subprobability. Its mass is $1-\TV(P,\widehat P)$. Sample identically from this part and couple the residual parts arbitrarily. This gives both required marginals and disagreement probability equal to the local total variation.

Run that coupling only while the histories agree, using identical policy draws on common histories. Let $E_j$ be first disagreement at transition $j$. If $A_j$ is agreement through $(s_j,a_j)$, then
\begin{equation*}
\Prob(E_j)=
\E\!\left[\ind\{A_j\}\TV\{P(\cdot\mid s_j,a_j),\widehat P(\cdot\mid s_j,a_j)\}\right]
\leq d_j.
\end{equation*}
Let $Q_t,\widehat Q_t$ be the prefix laws through $(s_t,a_t)$. The first-disagreement events are disjoint, so
\begin{equation*}
\TV(Q_t,\widehat Q_t)\leq\Prob\!\left(\bigcup_{j<t}E_j\right)
=\sum_{j<t}\Prob(E_j)\leq\sum_{j<t}d_j.
\end{equation*}
Combining with $\TV(Q_t,\widehat Q_t)\leq1$ gives $D_t$.

For any vector-valued $F$ with $\|F\|_2\leq C$,
\[
\|\E_PF-\E_QF\|_2
=\sup_{\|v\|_2=1}|\E_P\langle v,F\rangle-\E_Q\langle v,F\rangle|
\leq2C\TV(P,Q).
\]
Using $\|S_t r_t\|_\infty\leq(t+1)GR_b$,
\begin{align*}
\|g-\widehat g\|_2
&\leq\sum_{t=0}^{H-1}\left\|\int S_t r_t\,\mathrm d(Q_t-\widehat Q_t)\right\|_2,\\
&\leq2GR_b\sum_{t=0}^{H-1}(t+1)D_t
\leq2C_H\min\{1,e\}.
\end{align*}
With uniform local error $\epsilon$,
\begin{equation*}
b_\theta\leq 2GR_b\sum_{t=0}^{H-1}(t+1)\min\{1,t\epsilon\}.
\end{equation*}
For small $H\epsilon$, this is $2GR_b\epsilon H(H-1)(H+1)/3$. We do not claim this cubic horizon dependence is minimax optimal over policy-gradient MDPs.

\subsection{Prediction error and comparison-specific sensitivity}
\label{app:separation}

\paragraph{Proof of Corollary~\ref{cor:absolute}.}
Fix $\epsilon\in(0,1/2]$ and a two-step task with initial actions $a_0,a_1$, terminal states $g,b$, zero initial reward, and terminal reward $c\in(0,R_b]$ at $g$.  In the true world let
\[
P(g\mid a_0)=\frac12,
\qquad
P(g\mid a_1)=\frac12-\frac\epsilon2,
\]
whereas the learned world model satisfies
\[
\widehat P(g\mid a_0)=\frac12,
\qquad
\widehat P(g\mid a_1)=\frac12+\frac\epsilon2.
\]
The remaining probability is assigned to $b$, and the two worlds agree elsewhere.  Hence
\[
\sup_a\TV\!\left(P(\cdot\mid a),\widehat P(\cdot\mid a)\right)=\epsilon,
\]
but the learned world model assigns comparison $c\epsilon/2$ to replacing $a_0$ by $a_1$, while the true comparison is $-c\epsilon/2$.  Thus arbitrarily small uniform transition error reverses the action ranking.

Parameterize the two-action construction of Corollary~\ref{cor:absolute} by $\pi_\theta(a_1)=\sigma(\theta)=(1+\exp(-\theta))^{-1}$. Then
\begin{equation*}
J(\theta)=\frac c2-\frac{c\epsilon}2\sigma(\theta),\qquad
\Jh(\theta)=\frac c2+\frac{c\epsilon}2\sigma(\theta).
\end{equation*}

The same construction has a differential consequence. The learned world model ranks $a_1$ above $a_0$ while the true world ranks them in the opposite order; hence world-model-greedy planning is wrong without any parameter update. Increasing the logistic probability of $a_1$ strictly decreases the true return, so optimizing the predictive surrogate follows the wrong direction.

For the converse, let the next state be $(y,v)\in\{0,1\}^2$ and the terminal reward be $cy$, with $c\in(0,R_b]$. At the initial state, actions $a_0,a_1$ produce $y=1$ with probabilities $1/4,3/4$ in both worlds. The true world sets $v=0$ and the learned world sets $v=1$. Terminal dynamics and rewards agree. The two initial next-state laws have disjoint supports, so their total variation is one. Nevertheless every policy has the same value in both worlds: only its initial action affects reward, and the terminal action is immaterial. With logistic initial action probability,
\begin{equation*}
J(\theta)=\Jh(\theta)=c[1/4+\sigma(\theta)/2].
\end{equation*}
Values, action ordering, and policy gradients are all exact despite maximal transition discrepancy. This is why a total-variation-based sufficient gate need not characterize every valid use.

Qualification is comparison-specific. Consider three actions with terminal success probabilities
\begin{equation*}
p_0=1/4,\qquad
p_1\in[1/4-\epsilon/2,1/4+\epsilon/2],\qquad
p_2=3/4,
\end{equation*}
where $0<\epsilon\leq1/4$. Supply the upper-endpoint model and take the lower endpoint as the true world. Replacing $a_0$ by $a_2$ has contrast $c/2$ in every plausible world; replacing $a_0$ by $a_1$ has contrast $-c\epsilon/2$ in truth and $c\epsilon/2$ in the model. Thus a certificate for the first comparison does not certify the second, even within the same task and with the same predictive model.

\subsection{Why value equality and joint parameter updates are insufficient}

In a two-step task, initial actions $a_1,a_0$, in this order, lead to terminal reward $c$ with true success probabilities $1,0$ and learned success probabilities $0,1$. For the logistic policy $\pi_\theta(a_1)=\sigma(\theta)$,
\begin{equation*}
J(\theta)=c\sigma(\theta),\qquad
\Jh(\theta)=c[1-\sigma(\theta)].
\end{equation*}
At $\theta=0$ the values both equal $c/2$, but the gradients are $c/4$ and $-c/4$. Policy scores are bounded by one. Equality of scalar values at a point therefore provides no gradient certificate, even with bounded rewards and regular policies.

Freezing the world is also essential. Consider a one-action bandit with true reward $1/2$ and model reward $\widehat r_\theta=1/2+\theta$ for $|\theta|<1/4$. At $\theta=0$ model error is zero and the policy cannot change the true return. Nevertheless the total derivative of the model return is one. This derivative changes the model, not the policy. The main gradient theorem controls derivatives through a fixed predictive world; it does not certify unrestricted shared-parameter co-training. Shared-parameter methods require separating these derivative paths or proving an additional bound.

\subsection{Reward-model error}
\label{app:reward-error}

Let $\widehat r$ also be fixed with respect to $\theta$, and bounded in $[0,R_b]$. The learned gradient now uses
$\widehat F_\theta=\sum_t S_t\widehat r(s_t,a_t)$.
Adding and subtracting $\E_{\widehat P}F_\theta$ gives
\begin{equation*}
\|g-\widehat g\|_2
\leq 2GR_b\sum_{t=0}^{H-1}(t+1)D_t
 +G\sum_{t=0}^{H-1}(t+1)\E_{\widehat\mu_t^{\pi_\theta}}|r-\widehat r|.
\end{equation*}
Indeed, the first difference uses the common true-reward functional and prefix coupling; the second is bounded pointwise by
$\sum_t(t+1)G|r-\widehat r|$ and integrated under the learned occupancy. A uniform reward error $\epsilon_r$ yields the additional radius
$G\epsilon_r H(H+1)/2$.
Every gate must include this additional radius when rewards are learned. A sampled reward discrepancy is not a uniform bound unless separately certified.

\subsection{Sharp angular geometry}
\label{app:geometry}

We prove the sharp angular consequence the angular inequality above for Theorem~\ref{thm:gradient}.
Here $\cos(g,\widehat g)=\langle g,\widehat g\rangle/(\|g\|_2\|\widehat g\|_2)$.

Let $v\ne0$, $h=\|v\|_2$, and suppose $\|q-v\|_2\leq b<h$. Put $u=v/h$ and decompose $q=xu+y$ with $y\perp u$. The constraint is
$(x-h)^2+\|y\|_2^2\leq b^2$, so $x>0$. The largest possible squared tangent of the angle is
\[
\max_{x\in[h-b,h+b]}
 \frac{b^2-(x-h)^2}{x^2}.
\]
Writing $F(x)=[b^2-(x-h)^2]/x^2$, we have
\begin{equation*}
F'(x)=\frac{2(h^2-b^2-hx)}{x^3},\qquad
x_*=(h^2-b^2)/h,\qquad F(x_*)=\frac{b^2}{h^2-b^2}.
\end{equation*}
Since $x_*$ lies in $[h-b,h+b]$ and the derivative changes from positive to negative there,
\begin{equation*}
\frac{\langle q,v\rangle}{\|q\|_2\|v\|_2}
\geq\frac{\sqrt{h^2-b^2}}h.
\end{equation*}
For dimension at least two and $b>0$, choose a perpendicular $y$ attaining the boundary to obtain equality. If $b=0$ or the dimension is one, the cosine is one; the displayed lower bound still holds. If $b\geq h$, the uncertainty ball includes zero, and for $b>h$ it includes vectors pointing against $v$. This proves the exact robust threshold for an acute-angle guarantee.

\subsection{Optimal robust displacement and sharpness}
\label{app:robust}

For every displacement $v$ we prove the exact robust-support formula the robust-support identity above.

Suppose a differentiable objective has unknown gradient in the ball
$\{q:\|q-\widetilde g\|_2\leq b\}$ and satisfies the local smooth lower model
$J(\theta+v)-J(\theta)\geq\langle q,v\rangle-L\|v\|_2^2/2$.
For $v\ne0$, the adverse gradient is
$q=\widetilde g-bv/\|v\|_2$. Hence the robust-support identity above holds exactly; for $v=0$ both sides are zero.

Put $h=\|\widetilde g\|_2$ and let $\Phi(v)$ denote the right-hand side of the robust-support identity above. Cauchy--Schwarz and one-variable maximization give
\begin{align*}
\sup_{\|v\|_2=t}\Phi(v)&=(h-b)t-\tfrac L2t^2,\\
\sup_{v\in\mathbb R^d}\Phi(v)&=\sup_{t\geq0}\{(h-b)t-\tfrac L2t^2\}
=\frac{[h-b]_+^2}{2L}.
\end{align*}
For $h>b$, the maximizer is $v_*=(h-b)\widetilde g/(Lh)$; for $h\leq b$, it is $v_*=0$. The unconstrained maximum is attainable as a certified step only if the segment from $\theta$ to $\theta+v_*$ is feasible and lies in the smoothness region.

For a specified $v$, the quadratic
\begin{equation*}
f(\theta+w)=f(\theta)+\langle q,w\rangle-\frac L2\|w\|_2^2
\end{equation*}
attains the lower model. These local quadratics establish sharpness given gradient-ball and smoothness information. They are not asserted to be globally bounded MDP returns; a tighter guarantee exploiting further MDP structure is not ruled out.

A sufficient smoothness bound is available when policies are twice continuously differentiable and
$\|\nabla_\theta^2\log\pi_\theta(a\mid s)\|_2\leq K$ uniformly on a convex region. Differentiating the prefix expression for expected reward yields
\[
\nabla_\theta^2 J=
\E_P\sum_{t=0}^{H-1}r_t
 \left[S_tS_t^\top+\sum_{j=0}^t\nabla_\theta^2\log\pi_\theta(a_j\mid s_j)\right].
\]
The domination argument used for the first derivative extends using these bounds. Therefore
\begin{equation*}
L=R_b\left[
G^2\frac{H(H+1)(2H+1)}6+
K\frac{H(H+1)}2\right]
\end{equation*}
is sufficient. A smaller valid local $L$ improves the gate, but a numerical curvature estimate alone is not a guaranteed upper bound.

\subsection{Actual optimizers and latent interfaces}
\label{app:implementation}

The comparison certificate applies to a realized output of any predictive search procedure. The differential condition is useful when that procedure optimizes a parameterized behavior. It validates the proposed displacement; it does not require a synthetic-data training pipeline.

For differentiable policies, a proposed displacement need not be parallel to a policy-gradient estimate. If $\|g-\widetilde g\|_2\leq b$ and the proposed segment is $L$-smooth, then for every feasible $v$,
\begin{equation*}
J(\theta+v)-J(\theta)
\geq\langle\widetilde g,v\rangle-b\|v\|_2-\frac L2\|v\|_2^2.
\end{equation*}
This follows by inserting the gradient-ball support function into the smoothness inequality. It can certify a preconditioned, clipped, or otherwise proposed optimizer displacement without identifying it with the exact gradient. The estimate must target the score-function gradient of the frozen imagined world. If a practical estimate has an additional bias bounded by $\zeta$, replace $b$ by $b_\theta+\xi+\zeta$. Merely naming a critic or using automatic differentiation does not establish such a bound.

A common history interface need not require reconstructing every observation. The following bridge states precisely what is sufficient for a fixed latent policy interface.

\begin{proposition}[History-to-latent reliability bridge]\label{prop:latent}
Let $h$ be a true observable history, $s=f(h)$ a measurable, parameter-independent representation, and let the policy use only $s$. Start the latent model at the pushforward of the true initial history law. Write $K(\cdot\mid h,a)$ for the true conditional law of the next representation and $\widehat P(\cdot\mid f(h),a)$ for its learned prediction. Suppose, uniformly in admissible histories and actions,
\begin{equation*}
|r(h,a)-\widehat r(f(h),a)|\leq\epsilon_r,\qquad
\TV\{K(\cdot\mid h,a),\widehat P(\cdot\mid f(h),a)\}\leq\epsilon_p.
\end{equation*}
Then every such policy satisfies
\begin{equation*}
|J-\Jh|\leq H\epsilon_r+
 \frac{H(H-1)}2R_b\epsilon_p.
\end{equation*}
Under the same fixed-interface score regularity as Theorem~\ref{thm:gradient}, its gradient discrepancy is at most
\begin{equation*}
\|g-\widehat g\|_2
\leq 2GR_b\sum_{t=0}^{H-1}(t+1)\min\{1,t\epsilon_p\}
 +\frac{GH(H+1)}2\epsilon_r.
\end{equation*}
The true representation process need not be Markov.
\end{proposition}

\begin{proof}
Lift the learned value to histories as $\widehat V_t(f(h))$. Subtract its learned Bellman recursion from the true history recursion and add the true conditional expectation of $\widehat V_{t+1}(f(h'))$. The residual is the reward difference plus
$(K-\widehat P)\widehat V_{t+1}$, bounded by
$\epsilon_r+(H-t-1)R_b\epsilon_p$. The remaining true conditional difference telescopes exactly as in Theorem~\ref{thm:planning}. Summing proves the latent value bound above.

For gradients, keep the complete true history in the construction while coupling the latent trajectories. On agreement of latent prefixes, the two action laws coincide. Couple the true next representation, conditional on its full history, with the learned latent transition by maximal coupling; disagreement probability is at most $\epsilon_p$. A regular conditional distribution of the next true history given its representation preserves the full true marginal. Standard Borel assumptions ensure these conditional kernels exist. Once prefixes separate, continue with any marginal-preserving coupling. The latent prefix discrepancy through time $t$ is at most $\min\{1,t\epsilon_p\}$.

Use the learned reward functional in both worlds. The true gradient's reward-replacement error is at most
$G\epsilon_r\sum_t(t+1)$. The remaining common functional depends only on latent prefixes; the score is $\nabla\log\pi_\theta(a\mid f(h))$ in the true process and the same function of $(s,a)$ in the learned process. Applying the prefix argument proves the latent gradient bound above. Parameter independence of $f$ and the supplied world is required for this common functional.
\end{proof}

The premise includes representation aliasing: histories with the same representation must all have predictions close enough to the supplied kernel. Passive reconstruction loss does not establish it. If the representation admits an exact Markov kernel, finite-state row verification can be applied to that kernel; otherwise a latent-row average needs an additional uniform aliasing bound. Encoder updates or shared-parameter model updates require rechecking the interface and derivative premises. These distinctions allow the theory to be used with learned representations without assuming their sufficiency by definition.

\subsection{Dimension-free imagined-gradient sampling}

A sampling radius for the differential certificate is
\begin{equation*}
\xi=\frac{C_H}{\sqrt N}\left(1+\sqrt{2\log(1/\delta)}\right)
\end{equation*}
where $N$ is the number of independent imagined trajectories and $\delta\in(0,1)$.

\begin{proposition}[Bounded-vector sampling]\label{prop:gradient-sampling}
Condition on a fixed policy and learned world. Let $Y_1,\ldots,Y_N$ be independent samples of $\widehat F_\theta$ with $\|Y_i\|_2\leq C_H$, and let $\widetilde g=N^{-1}\sum_iY_i$. Then $\|\widetilde g-\widehat g\|_2\leq\xi$ with probability at least $1-\delta$, for $\xi$ in the stated sampling radius. The radius may be truncated at $2C_H$.
\end{proposition}

\begin{proof}
Independence and centering cancel cross terms:
\[
\E\|\widetilde g-\widehat g\|_2^2
=\frac1{N^2}\sum_i\E\|Y_i-\widehat g\|_2^2
\leq\frac{C_H^2}{N}.
\]
Jensen gives $\E\|\widetilde g-\widehat g\|_2\leq C_H/\sqrt N$.
Replacing one sample changes this norm by at most $2C_H/N$. The bounded-differences inequality therefore bounds the probability of exceeding its expectation by $t$ by
$\exp[-Nt^2/(2C_H^2)]$.
Take $t=C_H\sqrt{2\log(1/\delta)/N}$.
Finally $\|\widetilde g\|_2,\|\widehat g\|_2\leq C_H$, giving the deterministic truncation.
\end{proof}

The bound has no parameter-dimension factor because the assumed norm bound already controls the entire vector. It applies to independent trajectory estimators of the stated gradient, not automatically to replay-correlated minibatches, bootstrapped actor losses, or biased value-gradient estimators. An additional estimator-bias radius must be added for those alternatives.

\section{Statistical Certification and Adaptive Verification}
\label{app:calibration}

\subsection{Calibrating decision-relevant prediction error}
Let $x=(z,\Mh,\pi_0,\pi_1)$ be a complete decision request, with
$e(x)=|\widehat\Gamma(x)-\Gamma_\M(x)|$. An independently fitted score $u$ predicts this nonnegative error. Conditional on that fit, assume exchangeability of the calibration pairs $(x_j,e_j)$ and one future pair. For $\alpha\in(0,1)$ put
\begin{equation*}
q_j=e_j-u(x_j),\quad
k=\lceil(n+1)(1-\alpha)\rceil,\quad
U_\alpha(x)=[u(x)+q_{(k)}]_+.
\end{equation*}
Set $q_{(k)}=+\infty$ if $k>n$. The request includes the actual model, reference, candidate selection procedure and continuation; changing their distribution is a change of the calibration population.

\begin{proposition}[Calibrated contrast and false-promotion risk]\label{prop:conformal}
Under the preceding exchangeability assumption,
$\Prob\{e(x)>U_\alpha(x)\}\leq\alpha$. Suppose
$\Prob\{\widehat\Gamma(x)<S(x)-\xi(x)\}\leq\delta$. Define
\begin{equation*}
T(x)=S(x)-U_\alpha(x)-\xi(x).
\end{equation*}
Then
\begin{equation*}
\Prob\{T(x)>0,\ \Gamma_\M(x)\leq0\}\leq\alpha+\delta.
\end{equation*}
\end{proposition}
\begin{proof}
After independent randomized tie-breaking, exchangeability makes the future residual's rank $K$ uniform on $\{1,\ldots,n+1\}$. Therefore
\begin{align*}
\Prob\{q>q_{(k)}\}
&\leq\Prob\{K>k\}\\
&=\sum_{j=k+1}^{n+1}\Prob\{K=j\}
=\frac{n+1-k}{n+1}\leq\alpha.
\end{align*}
Since $q=e-u(x)$ and $U_\alpha=[u(x)+q_{(k)}]_+$, the event $e>U_\alpha$ is contained in $q>q_{(k)}$. On its complement and on $\widehat\Gamma\geq S-\xi$,
\begin{align*}
\Gamma_\M
&=\widehat\Gamma+(\Gamma_\M-\widehat\Gamma)\\
&\geq\widehat\Gamma-|\Gamma_\M-\widehat\Gamma|\\
&\geq S-\xi-U_\alpha=T.
\end{align*}
Hence
\begin{align*}
\{T>0,\Gamma_\M\leq0\}
&\subseteq\{e>U_\alpha\}\cup\{\widehat\Gamma<S-\xi\},\\
\Prob\{T>0,\Gamma_\M\leq0\}
&\leq\Prob\{e>U_\alpha\}+\Prob\{\widehat\Gamma<S-\xi\}\\
&\leq\alpha+\delta.
\end{align*}
\end{proof}
This is split conformal calibration applied to a decision-relevant error target \citep{angelopoulos2023conformal}. It does not identify the true error from a single transition. A valid upper label such as $\epsint(\pi_0)+\epsint(\pi_1)$ may replace $e$ when exact contrast errors are unavailable, subject to the labeling conditions below.
\paragraph{Calibration versus raw uncertainty scores.}
The certificate concerns a calibrated upper envelope of the decision-relevant error
$e(x)=|\widehat\Gamma(x)-\Gamma_\M(x)|$. A confidence, disagreement, entropy, or
critic score is therefore not itself an error certificate: it becomes usable in the
gate only after an argument establishes the required upper-error guarantee for the
request population under consideration.
\subsection{Split conformal coverage and the exact promotion event}

Condition on the training data used to fit $u$. Let $q_1,\ldots,q_n,q$ be the exchangeable calibration and future residuals, and let $K$ denote the rank of $q$ after independent randomized tie-breaking. For $k\leq n$,
\begin{align*}
\Prob(K=j)&=(n+1)^{-1},\qquad j=1,\ldots,n+1,\\
\Prob\{q>q_{(k)}\}&\leq\Prob(K>k)
=\frac{n+1-k}{n+1}\leq\alpha.
\end{align*}
Since $q=e-u(x)$ and $U_\alpha(x)=\max\{0,u(x)+q_{(k)}\}$,
\begin{equation*}
\{e>U_\alpha(x)\}\subseteq\{q>q_{(k)}\},\qquad
\Prob\{e>U_\alpha(x)\}\leq\alpha.
\end{equation*}
For $k=n+1$, set $q_{(k)}=+\infty$; coverage is then immediate.

Define the events
$E=\{e\leq U_\alpha\}$ and
$F=\{\widehat\Gamma\geq S-\xi\}$.
On $E\cap F$, positive $T$ implies strictly positive true gain. Therefore
\begin{equation*}
\{T>0,\ \Gamma_\M(x)\leq0\}
\subseteq E^c\cup F^c.
\end{equation*}
Consequently, without independence of $E,F$,
\begin{equation*}
\Prob\{T>0,\ \Gamma_\M(x)\leq0\}
\leq\Prob(E^c)+\Prob(F^c)\leq\alpha+\delta.
\end{equation*}

In contrast, if $p=\Prob(T>0)>0$, the argument gives at most
\begin{equation*}
\Prob\{\Gamma_\M(x)\leq0\mid T>0\}
\leq\min\{1,(\alpha+\delta)/p\},
\end{equation*}
not $\alpha+\delta$. A generic marginal certificate can fail exclusively on a selected minority: let $X$ be Bernoulli with mean $\alpha$, let $e(X)=X$, and set $U(X)=0$. Marginal coverage is $1-\alpha$, yet conditional miscoverage given $X=1$ is one. This demonstrates why marginal validity alone cannot justify selection-conditional validity; it is not a claim that this particular $U$ was produced by split conformal prediction.

For multiple future snapshots chosen before observing calibration outcomes, one may compute each marginal certificate at level $\alpha/m$ for a fixed pool of size $m$. A union bound gives simultaneous coverage at least $1-\alpha$ over the pool, permitting arbitrary subsequent selection within it. Each pair must still have the required marginal exchangeability; the pool cannot be manufactured adaptively from the same calibration residuals without further analysis.

\subsection{Noisy labels and finite strata}

If exact errors are unavailable, suppose an identically applied labeling procedure produces exchangeable upper labels $\overline e_i$, including a hypothetical future $\overline e$, and
$\Prob(e>\overline e)\leq\beta$.
Calibrate the residuals $\overline e_i-u(x_i)$ instead. Conformal coverage of $\overline e$ and a union bound then yield
\begin{equation*}
\Prob\{e>U_\alpha(x)\}\leq\alpha+\beta.
\end{equation*}
This statement requires upper labels and their separate validity bound, not merely unbiased estimates. A pixel discrepancy or mean state error may instead define a different target; conformal coverage of that target does not bound $e$ without a proven bridge.

For a prespecified finite partition $c(x)\in\{1,\ldots,m\}$, calibrate separately within each stratum. Conditional on a future stratum $h$, the fitted score, and its calibration count $n_h$, assume the corresponding residuals are exchangeable. The same rank argument with
$k_h=\lceil(n_h+1)(1-\alpha)\rceil$
gives
\begin{equation*}
\Prob\{e\leq U_h(x)\mid c(x)=h\}\geq1-\alpha.
\end{equation*}
The partition is fixed independently of calibration errors. Conditioning on arbitrary within-stratum promotion decisions is still not covered.

\subsection{Fresh finite-state verification}
\label{app:fresh}

This construction supplies a fully observable certificate under stronger access assumptions. Fix a finite task with $D=|\mathcal S|$ and $m=|\mathcal S||\mathcal A|$ state--action rows. After choosing the task, obtain $n$ independent next-state samples per row from a generative oracle. For learned rewards, also obtain bounded reward samples in $[0,R_b]$ with the correct conditional means. Samples within each row are conditionally independent given the pre-audit history. Across-row independence is unnecessary. Write $\overline P,\overline r$ for the empirical kernels and means and define
\begin{equation*}
a_n=\sqrt{\frac{D\log2+\log(4m/\alpha)}{2n}},
\qquad
c_n=R_b\sqrt{\frac{\log(4m/\alpha)}{2n}}.
\end{equation*}
For each row set
\begin{equation*}
u(s,a)=\min\{1,\TV(\overline P,\widehat P)+a_n\},
\qquad
v(s,a)=\min\{R_b,|\overline r-\widehat r|+c_n\}.
\end{equation*}
The symbols $s,a$ in row arguments denote a state and an action.

\begin{proposition}[Simultaneous fresh-audit certificate]\label{prop:fresh}
Conditional on the pre-audit history, with probability at least $1-\alpha$, every row obeys
$\TV(P,\widehat P)\leq u$ and $|r-\widehat r|\leq v$.
The statement holds simultaneously for all bounded world-model candidates, including candidates fitted using the audit. If $u_*=\max_{s,a}u(s,a)$ and $v_*=\max_{s,a}v(s,a)$, then every policy has
\begin{equation*}
\epsint(\pi)\leq Hv_*+\frac{H(H-1)}2R_bu_*,
\end{equation*}
and every policy satisfying the score bound has
\begin{equation*}
\|g-\widehat g\|_2
\leq
2GR_b\sum_{t=0}^{H-1}(t+1)\min\{1,tu_*\}
 +\frac{GH(H+1)}2v_*.
\end{equation*}
\end{proposition}

\begin{proof}
For any fixed subset of the $D$ next states, its empirical probability is the mean of Bernoulli variables. Hoeffding bounds absolute error above $a_n$ by $2\exp(-2na_n^2)$. Union over at most $2^D$ subsets and $m$ rows gives failure probability at most $\alpha/2$ for
$\TV(P,\overline P)\leq a_n$.
A separate Hoeffding and row union bound give
$|r-\overline r|\leq c_n$ with failure probability at most $\alpha/2$.

On this common event, triangle inequalities give the row envelopes above. The event constrains only the true and empirical kernels; hence the inequalities hold for every supplied world model, including data-dependent ones. Clipping is valid because total variation is at most one and rewards differ by at most $R_b$. The uniform value inequality proves the audited value bound above. The reward-gradient bound and $d_t\leq u_*$ prove the audited gradient bound, uniformly over the admitted policies.
\end{proof}

This verifier costs $mn$ real row queries and does not infer total variation from prediction loss. Its strength is simultaneous validity after policy selection; its weakness is dependence on finite spaces and generative coverage. In large or continuous spaces it must be replaced by a justified structured confidence set or a direct return audit, not presented as computationally free.

\subsection{Shared evidence under adaptive predictive decisions}
\label{app:shared}
\paragraph{Proof of Theorem~\ref{thm:shared}.}
Let $E$ be the simultaneous event from Proposition~\ref{prop:fresh}. It constrains the true and empirical transition rows, and reward means when audited, independently of the candidate policies or learned world models. For every selected world model, let $u_i,v_i$ be the maxima of the transition and reward envelopes defined above. Then
\begin{equation*}
B_i=2Hv_i+R_bH(H-1)u_i,\qquad
|\Gamma_\M(x_i)-\widehat\Gamma_i|\leq B_i\quad\hbox{on }E.
\end{equation*}
For known shared rewards put $v_i=0$. The triangle inequalities hold simultaneously for models fitted on the audit as well as for independently fitted models.

Let $\mathcal H_i$ contain the audit, prior requests and observations, the current world model, and the selected policy pair, before estimation of the current contrast. For
$F_i=\{\widehat\Gamma_i\geq S_i-\xi_i\}$ assume
$\Prob(F_i^c\mid\mathcal H_i)\leq\delta_i$. Write $\mathcal I$ for the random set of accepted indices. On $E\cap F_i$, any accepted request satisfies
\begin{align*}
\Gamma_\M(x_i)
&\geq \widehat\Gamma_i-B_i\\
&\geq S_i-\xi_i-B_i
=T_i>0.
\end{align*}
Consequently, the bad-promotion event obeys
\begin{equation*}
\mathcal B=\bigcup_{i\geq1}\bigl(\{i\in\mathcal I\}\cap\{\Gamma_\M(x_i)\leq0\}\bigr)
\subseteq E^c\cup\bigcup_{i\geq1}F_i^c.
\end{equation*}
The conditional validity of the fresh estimate and the tower property give
\begin{align*}
\Prob(\mathcal B)
&\leq\Prob(E^c)+\Prob\!\left(\bigcup_{i\geq1}F_i^c\right)\\
&\leq\Prob(E^c)+\sum_{i\geq1}\Prob(F_i^c)\\
&=\Prob(E^c)+\sum_{i\geq1}
  \E\!\left[\Prob(F_i^c\mid\mathcal H_i)\right]\\
&\leq\alpha+\sum_{i\geq1}\delta_i.
\end{align*}
No conditional coverage of $E$ after selection is used. The argument requires the unconditional common event and conditional validity of each fresh estimation step.

For a selected pair, draw $N$ independent paired imagined episodes and let $S_i$ be the sample mean difference. Each difference lies in $[-HR_b,HR_b]$, so
\begin{equation*}
\widehat\Gamma_i\geq S_i-\xi_i,\qquad
\xi_i=HR_b\sqrt{\frac{2\log(1/\delta_i)}{N}}
\end{equation*}
with conditional probability at least $1-\delta_i$. Pairing within an index is allowed; pairs must be independent. If the candidate is selected from $K$ prespecified candidates using these same samples, replace $\delta_i$ by $\delta_i/K$ and union the bounds. Unrestricted selection requires fresh evaluation or a justified uniform estimator. This sampling restriction is separate from reuse of the model audit.

For Theorem~\ref{thm:closed-loop}, the uniform row event bounds every conditional $Q_t^0$ through~\eqref{eq:conditional-radius}. Exact world-model evaluation then gives simultaneous local certificates at all states. With sampled local evaluation, either use a uniform finite-state correction or conditional fresh estimates along the actual trajectory. The latter controls bad local decisions along that trajectory; it must not be presented as a uniform bound over all unvisited histories.

For a fixed task archive, allocate audit risk $\alpha_z$ to task $z$. Total false-promotion risk is at most $\sum_z\alpha_z+\sum_i\delta_i$, with each stationary task's model risk charged once. If a task's transition kernel drifts by at most $\beta$ in uniform TV, replace $u_i$ by $\min\{1,u_i+\beta\}$. Reward drift analogously increases $v_i$. These statements follow from triangle inequalities. A new task or unbounded drift requires new coverage evidence.

\subsection{Amortized evidence cost at positive decision margins}
\label{app:amortized}
Assume shared rewards, $H\geq2$ and $R_b>0$. Set
\begin{equation*}
K=R_bH(H-1),\qquad q_i=\max_{s,a}\TV(P,\widehat P_i),\qquad
a_i=\widehat\Gamma_i-Kq_i.
\end{equation*}
The quantity $a_i$ is a sufficient margin involving the unknown true model error, not an observable gain estimate.

\begin{proposition}[Amortized verification]\label{prop:amortized-cost}
On the row-confidence event, $u_i\leq q_i+2a_n$ simultaneously for every world model. Exact world-model evaluation therefore gives~\eqref{eq:margin-audit-price}. The sample size~\eqref{eq:amortized-size} certifies every request with $a_i/K\geq s$. For a sampled estimate satisfying $|S_i-\widehat\Gamma_i|\leq\xi_i$, the corresponding bound is
\begin{equation*}
T_i\geq a_i-2Ka_n-2\xi_i.
\end{equation*}
\end{proposition}
\begin{proof}
For every row, the triangle inequality gives
\begin{equation*}
\TV(\overline P,\widehat P_i)+a_n
\leq\TV(P,\widehat P_i)+2a_n.
\end{equation*}
Maximization and clipping preserve this upper bound.
With $S_i=\widehat\Gamma_i$ and $\xi_i=0$,
$T_i=\widehat\Gamma_i-Ku_i\geq a_i-2Ka_n$.
Condition~\eqref{eq:amortized-size} is exactly $2a_n<s$.
If $a_i/K\geq s$, it follows that $T_i\geq K(s-2a_n)>0$.
For sampled evaluation, $S_i\geq\widehat\Gamma_i-\xi_i$ gives the sampled-margin inequality above.
\end{proof}
The real evidence cost is $mn$, instead of $Nmn$ for repeating that row audit before each of $N$ comparisons. Imagined evaluation and candidate search remain additional costs. Positive margins are a premise; their existence is not implied by confidence coverage. The statement is deterministic on the common event, so it remains valid for adaptive requests meeting that premise. A post-selection random margin cannot be treated as a prespecified unconditional power guarantee. When $H=1$ and rewards are shared, transitions do not affect returns and no transition audit is required.

\subsection{Stopping an audit at a positive margin}
\label{app:stopping}
For adaptively chosen audit size, use
\begin{equation*}
\alpha_n=\frac{\alpha}{n(n+1)},\qquad
a_n=\sqrt{\frac{D\log2+\log(4m/\alpha_n)}{2n}}.
\end{equation*}
For each row take successive prefixes of an independent sample stream. Since $\sum_{n\geq1}\alpha_n=\alpha$, one event covers all integer sample sizes and permits stopping at the first positive margin. The log factor now grows with $n$; the fixed-size bound is not an anytime bound.

For a fixed request with exact imagined contrast and positive sufficient margin $a=\widehat\Gamma-Kq$, fixed-size verification needs at most any integer satisfying
\begin{equation*}
n>\frac{2K^2}{a^2}[D\log2+\log(4m/\alpha)].
\end{equation*}
The right side is a sufficient audit budget, not a bound on how quickly learning reduces the model error $q$.

\subsection{A lower bound on the evidence needed for promotion}
\label{app:evidence}

We prove the necessary query budget~\eqref{eq:verification-lower}, with Bernoulli relative entropy $\operatorname{kl}$ defined below.

\begin{proposition}[Evidence required by a useful gate]\label{prop:verification-lower}
For every $\gamma\in(0,1/4]$ there are two two-step worlds and one fixed supplied world model such that one fixed candidate decision is beneficial in one and harmful in the other. If the audit's only world-dependent observations are row-query answers, promotion probability at least $1-\delta$ in the first world and at most $\delta$ in the second requires~\eqref{eq:verification-lower}, for $0<\delta<1/2$, even with adaptive query selection.
\end{proposition}

We prove Proposition~\ref{prop:verification-lower}. All side information at the start of the audit is identical under the two hypotheses; only new oracle answers distinguish them. Use horizon two and states $\{s,g,b\}$, with zero reward at $s,b$ and terminal reward $c\in(0,R_b]$ at $g$. There are two actions. In both worlds $P(g\mid s,a_0)=1/2$; in the plus and minus worlds, respectively, $P(g\mid s,a_1)=1/2+\gamma$ and $1/2-\gamma$. All other rows agree. The supplied world model is the plus world, and the same policy family is used in both:
$\pi_\theta(a_1\mid s)=\sigma(\theta)$.
Then
\begin{equation*}
J_+(\theta)=c[1/2+\gamma\sigma(\theta)],\qquad
J_-(\theta)=c[1/2-\gamma\sigma(\theta)].
\end{equation*}
Replacing $a_0$ by $a_1$ has contrasts $c\gamma$ and $-c\gamma$, respectively. Thus a safe and useful promotion decision must distinguish the two worlds. The logistic parameterization merely interpolates between the same two actions.

For $p,q\in(0,1)$ define
\[
\operatorname{kl}(p,q)=p\log(p/q)+(1-p)\log[(1-p)/(1-q)].
\]
Use the usual limiting conventions at zero and one.
One informative query, at $(s,a_1)$, has relative entropy
\begin{equation*}
\operatorname{kl}(1/2+\gamma,1/2-\gamma)
=2\gamma\log\frac{1+2\gamma}{1-2\gamma}
\leq16\gamma^2.
\end{equation*}
For the inequality use
$\log[(1+x)/(1-x)]\leq2x/(1-x)$ at $x=2\gamma\leq1/2$.
Queries at other rows have zero relative entropy.

Let an audit make at most $n$ adaptively chosen row queries; include its independent randomization in the transcript $Y_{1:n}$. Conditional on $\mathcal H_{t-1}=\sigma(Y_{1:t-1})$, its selection kernel is identical under both hypotheses. Writing $\Prob_+^n,\Prob_-^n$ for the transcript laws, the relative-entropy chain rule yields
\begin{align*}
D_{\rm KL}(\Prob_+^n\|\Prob_-^n)
&=\sum_{t=1}^n\E_+\!\left[
 D_{\rm KL}\!\left(
 \Prob_+(Y_t\mid\mathcal H_{t-1})
 \middle\|
 \Prob_-(Y_t\mid\mathcal H_{t-1})
 \right)\right]\\
&\leq\sum_{t=1}^n
 \operatorname{kl}(1/2+\gamma,1/2-\gamma)\\
&\leq16n\gamma^2.
\end{align*}
If the audit stops early, append uninformative queries to reach $n$; this changes neither the decision nor the divergence bound.

Writing $p$ and $q$ for the plus and minus promotion probabilities, data processing for the binary decision gives transcript divergence at least $\operatorname{kl}(p,q)$. This follows by partitioning the likelihood-ratio integral over promotion and rejection and applying Jensen on each part. Since $p\geq1-\delta>\delta\geq q$, Bernoulli relative entropy is increasing in $p$ and decreasing in $q$ on this region; consequently
\begin{align*}
16n\gamma^2
&\geq D_{\rm KL}(\Prob_+^n\|\Prob_-^n)\\
&\geq\operatorname{kl}(p,q)\\
&\geq\operatorname{kl}(1-\delta,\delta),
\end{align*}
which proves~\eqref{eq:verification-lower}. The bound is for evidence needed by a safe, high-power promotion test, not a general lower bound on world-model training. It shows that an arbitrarily small unresolved intervention margin cannot be certified from a fixed amount of real evidence.

Together with the sufficient audit bound in
Appendix~\ref{app:amortized}, this establishes the matching
$\Theta(\gamma^{-2})$ dependence on the unresolved decision margin. We do not claim
matching state-space or horizon dependence: those factors arise from the particular
simultaneous finite-state verifier used for the upper bound.
\subsection{Direct return verification for general interfaces}

A finite-state confidence set is not necessary if true episodic evaluation is available. Choose policies $\pi,\pi'$ before collecting fresh evaluations. Let $Z_i$ be the difference of their returns in independent episode pairs. Pairing within an episode index is allowed, and $Z_i\in[-HR_b,HR_b]$. Then
\begin{equation*}
J(\pi')-J(\pi)
\geq \overline Z-HR_b\sqrt{\frac{2\log(1/\alpha)}n}
\end{equation*}
with probability at least $1-\alpha$, by one-sided Hoeffding. A positive lower bound directly certifies the candidate improvement, including for black-box planners and history-based agents.

This audit is more general but can be expensive: it validates each candidate with real episodes instead of amortizing model reliability across many imagined candidates. Its false-positive event, conditional on pre-audit choice, fits Theorem~\ref{thm:sequential}. If candidates are selected using these same evaluations, use a simultaneous finite-pool correction or fresh evaluation again.

\subsection{Elementary concentration tools}

For completeness, both concentration inequalities used above follow from a bounded exponential moment. If $X\in[a,b]$, let
$K(\lambda)=\log\E\exp[\lambda(X-\E X)]$.
Under the exponentially tilted law, $K''(\lambda)$ is the variance of $X$, bounded by $(b-a)^2/4$: center $X$ at $(a+b)/2$ and use that variance is at most the corresponding second moment. Since $K(0)=K'(0)=0$, integrating twice gives
\begin{equation*}
\E\exp[\lambda(X-\E X)]
\leq\exp\{\lambda^2(b-a)^2/8\}.
\end{equation*}
For independent variables, multiply these bounds. Markov's inequality for the exponential and optimization over $\lambda>0$ yield
$\Prob\{\overline X-\E\overline X\geq t\}\leq
\exp[-2nt^2/(b-a)^2]$.
Apply the same argument to $-X$ and add probabilities for a two-sided bound.

If a function $f(X_1,\ldots,X_n)$ of independent variables changes by at most $c_i$ when only $X_i$ changes, reveal the variables sequentially. The resulting Doob martingale difference has conditional mean zero and conditional range of length at most $c_i$: averaging over the remaining independent variables preserves the coordinate-wise difference bound. Applying the bounded exponential-moment inequality above conditionally and iterating gives
\begin{equation*}
\Prob\{f-\E f\geq t\}
\leq\exp\!\left[-\frac{2t^2}{\sum_i c_i^2}\right].
\end{equation*}
Using $c_i=2C_H/N$ proves the tail step in Proposition~\ref{prop:gradient-sampling}. These arguments also hold conditional on a pre-audit history when the required conditional independence is satisfied.

\section{Adaptive Co-learning and Resource Accounting}
\label{app:sequential}

\begin{theorem}[Adaptive fresh verification]\label{thm:sequential}
Let $\mathcal H_{k-1}$ contain all information before round $k$'s fresh audit, including its task choice. If the combined conditional failure probability is at most $\varepsilon_k$ and only positive certified comparisons are executed, let $\mathcal B$ denote the event of at least one non-improving execution in this sequence. Then
\begin{equation*}
\Prob(\mathcal B)
 \leq\sum_{k\geq1}\varepsilon_k,\qquad
\E N_K\leq\sum_{k=1}^K\varepsilon_k,
\end{equation*}
where $N_K$ counts non-improving executions through round $K$.
\end{theorem}

\paragraph{Proof of Theorem~\ref{thm:sequential}.}
Each bad execution implies audit or estimation failure. Conditional expectation and the tower property give its marginal bound; countable subadditivity and linearity give the two conclusions. Unlike Theorem~\ref{thm:shared}, this permits arbitrary task changes, but pays for fresh validity.

\subsection{Filtration and repeated promotion}

At round $k$, $\mathcal H_{k-1}$ contains previous world-model fits, policies, generator choices, and audit outcomes, together with the present task choice made before its fresh random samples. Let $E_k$ be the event on which all current world-model-error and sampling bounds used by the gate hold. Assume
$\Prob(E_k^c\mid\mathcal H_{k-1})\leq\varepsilon_k$,
where $\varepsilon_k$ is deterministic or predictable with a deterministic summable upper envelope. For simplicity the main theorem uses deterministic $\varepsilon_k$.

Let $I_k$ indicate that the algorithm accepts a candidate behavior and its true target return does not increase. The deterministic margin theorem gives $I_k\leq\ind\{E_k^c\}$. Consequently
\begin{align*}
\E[I_k\mid\mathcal H_{k-1}]
&\leq\E[\ind\{E_k^c\}\mid\mathcal H_{k-1}]
=\Prob(E_k^c\mid\mathcal H_{k-1})\leq\varepsilon_k,\\
\E\sum_{k=1}^K I_k
&=\sum_{k=1}^K\E\!\left[
  \E(I_k\mid\mathcal H_{k-1})\right]
\leq\sum_{k=1}^K\varepsilon_k.
\end{align*}
Moreover, the event of at least one bad accepted comparison is contained in $\bigcup_{k\geq1}E_k^c$, whence
\begin{align*}
\Prob(\mathcal B)
&\leq\Prob\!\left(\bigcup_{k\geq1}E_k^c\right)
\leq\sum_{k\geq1}\Prob(E_k^c)\\
&=\sum_{k\geq1}\E\!\left[
 \Prob(E_k^c\mid\mathcal H_{k-1})\right]
\leq\sum_{k\geq1}\varepsilon_k.
\end{align*}
Choosing
$\varepsilon_k=\varepsilon/[k(k+1)]$
protects every round with total risk at most $\varepsilon$.

Proposition~\ref{prop:fresh} supplies conditional world-model validity after adaptive task choice; the conditional sampling inequality above supplies simulation accuracy. Their error budgets add. This argument never conditions a marginal guarantee on a data-dependent promotion event. Between rounds all world models and task choices may change; this particular argument restores validity by fresh verification. For repeated uses on one fixed task, Theorem~\ref{thm:shared} instead retains a simultaneous world-confidence event and charges its failure probability only once.

\subsection{Fixed-objective progress and stationarity}
\label{app:stationarity}

The main-text contrast telescope~\eqref{eq:decision-progress} requires no derivatives. For differentiable predictive policy search, the local certificate also yields the following stationarity specialization.
The contrast telescope controls expected return rather than samplewise episode
rewards. Because $J(\pi)\leq HR_b$, any fixed positive certified margin can occur
only finitely many times along a sequence of accepted updates; this is the content
of~\eqref{eq:large-promotions}. The statement does not bound the waiting time,
number of rejected proposals, or real samples required between accepted updates.
On the event of simultaneous validity, enumerate accepted policy updates by $i=0,\ldots,n-1$. Model-only steps between them do not change the policy. Assume they optimize one fixed bounded objective $J$, have common smoothness $L$, and total gradient radius
$b_i\leq\kappa h_i$, where
$h_i=\|\widetilde g_i\|_2$ and $0\leq\kappa<1$.
Use $\eta=(1-\kappa)/L$ and require each segment to remain in the admissible region. Theorem~\ref{cor:safe-update} gives
\begin{equation*}
J(\theta_{i+1})-J(\theta_i)
\geq\frac{(1-\kappa)^2}{2L}h_i^2.
\end{equation*}
Since $\|\nabla J(\theta_i)\|_2\leq(1+\kappa)h_i$, summing yields
\[
\frac{(1-\kappa)^2}{2L(1+\kappa)^2}
 \sum_{i=0}^{n-1}\|\nabla J(\theta_i)\|_2^2
\leq J(\theta_n)-J(\theta_0)
\leq\sup_\theta J(\theta)-J(\theta_0).
\]
Therefore
\begin{equation*}
\min_{0\leq i<n}\|\nabla J(\theta_i)\|_2^2
\leq\frac1n\sum_{i=0}^{n-1}\|\nabla J(\theta_i)\|_2^2
\leq\frac{2L(1+\kappa)^2[\sup_\theta J(\theta)-J(\theta_0)]}{(1-\kappa)^2n}.
\end{equation*}
If infinitely many updates are accepted, boundedness of $J$ also gives
\begin{equation*}
\sum_{i=0}^\infty\|\nabla J(\theta_i)\|_2^2<\infty
\quad\Longrightarrow\quad
\lim_{i\to\infty}\|\nabla J(\theta_i)\|_2=0.
\end{equation*}
Existence of infinitely many certified updates is not assumed to follow from validity alone.

A fixed distribution over tasks can be incorporated in the initial state, defining one expected-return objective. By contrast, certifying improvement only for a selected task does not certify improvement of the task average. If objectives change to $J_i$ with
$\sup_\theta|J_{i+1}(\theta)-J_i(\theta)|\leq v_i$,
then
\begin{equation*}
\sum_{i=0}^{n-1}[J_i(\theta_{i+1})-J_i(\theta_i)]
\leq HR_b+\sum_{i=1}^{n-1}v_{i-1}.
\end{equation*}
To verify this, telescope while adding
$J_{i-1}(\theta_i)-J_i(\theta_i)$ at every internal point. The endpoint range is at most $HR_b$ and each internal discrepancy is at most its stated drift bound. Substituting this right-hand side in the preceding calculation gives a drift-corrected average bound for $\|\nabla J_i(\theta_i)\|_2^2$. Without drift control it is not a convergence theorem for a fixed target.

\subsection{What reliability does and does not allocate}

The frontiers in~\eqref{eq:frontiers}--\eqref{eq:agent-frontier} are disjoint: $\FW$ has nonpositive certified margin, while $\FA$ has positive margin. Membership in $\FW$ requests evidence or model refinement for a promising comparison. Membership in $\FA$ authorizes the candidate behavior. Neither identifies failure ownership or supplies an environment-wide guarantee.

A model can be reliable while its imagined return estimate is too noisy. Increasing the rollout count reduces $\xi$ without changing the predictive kernel. A request with no statistically resolved apparent benefit is deferred. This third possibility prevents a two-way routing rule from treating lack of evidence as proof of model ignorance or agent incompetence.

To define the next task, a generator may prioritize uncertainty, measured progress, diversity, or intervention coverage. The theory constrains the use of a selected task after verification; it does not claim an optimal open-ended task generator.

\subsection{Deterministic and stochastic precedence accounting}

Let $\mathcal Z$ now be a finite task archive. For each task fix nonnegative integers $w_z,a_z$. A productive model update decrements $w_z$ by one. A productive policy update decrements $a_z$ by one only once $w_z=0$. Updates have no transfer and no regression. An update outside these conditions is ineffective. These assumptions are an explicit learning-response abstraction, not consequences of a value-error bound.

\begin{theorem}[Exact precedence accounting]\label{thm:allocation}
Let $B^*=\sum_z(w_z+a_z)$. Every completed schedule has length
\begin{equation*}
B=B^*+m,
\end{equation*}
where $m$ is its number of ineffective updates. An eligible work-conserving schedule attains $B^*$.

More generally, suppose an eligible model trial on task $z$ succeeds independently with probability $p_z>0$ at cost $c_z>0$, and an eligible policy trial succeeds with probability $q_z>0$ at cost $d_z>0$. Each success decrements its corresponding integer by one. Every schedule that uses only eligible trials and continues until all tasks finish has expected cost
\begin{equation*}
B^*=\sum_z\left(\frac{c_z w_z}{p_z}+\frac{d_z a_z}{q_z}\right).
\end{equation*}
Extra ineffective trials add their expected costs.
\end{theorem}

\begin{proof}
For the deterministic case, let $\Phi_k=\sum_z(w_z(k)+a_z(k))$ be the remaining work after $k$ updates, and let $I_k$ indicate that update $k$ is ineffective. At completion time $B$,
\begin{equation*}
\Phi_0=B^*,\qquad\Phi_B=0,\qquad
\Phi_{k-1}-\Phi_k=1-I_k.
\end{equation*}
Writing $m=\sum_{k=1}^B I_k$, summation yields
\begin{equation*}
B^*=\sum_{k=1}^B(\Phi_{k-1}-\Phi_k)=B-m,\qquad B=B^*+m.
\end{equation*}
Eligible schedules have $m=0$ and attain $B^*$.

For the stochastic case, attach independent Bernoulli sequences to each task and stage. A nonanticipating schedule reveals the next unused entry only when that stage is attempted. Let $U_z,V_z$ be the numbers of eligible model and policy trials required to exhaust the two quotas. Geometric waiting times give
\begin{equation*}
\E U_z=w_z/p_z,\qquad\E V_z=a_z/q_z.
\end{equation*}
Every completing eligible schedule reveals exactly these prefixes, irrespective of interleaving. If $C$ is its total cost, then
\begin{equation*}
C=\sum_z(c_zU_z+d_zV_z),\qquad
\E C=\sum_z\left(\frac{c_zw_z}{p_z}+\frac{d_za_z}{q_z}\right)=B^*.
\end{equation*}
Ineffective trials neither reveal a productive entry nor change a quota; their costs add pathwise.
\end{proof}

Thus order inside an eligible frontier need not be uniquely optimal: the theorem characterizes a class of schedules with identical completion cost in this abstraction. It does not prove dominance over every scalar score, because a scalar rule can itself respect precedence. Cross-task transfer, warm-start benefits of early policy training, uncertain learning curves, forgetting, and model updates that never improve certification all change the allocation problem.

If every harmful false promotion in this abstraction wastes at most $c$ units of work, Theorem~\ref{thm:sequential} bounds its expected wasted work through round $K$ by $c\sum_{k=1}^K\varepsilon_k$. This is the cost of erroneous certification alone. Conservative rejection, audit cost, and ordinary optimization failure are additional costs and are not hidden inside that bound.

\section{Controlled Learned-World-Model Experiment}
\label{app:experiment-details}
\subsection{Experimental execution}
\label{app:exact-exp}\label{app:study-matrix}\label{app:empirical-gate}\label{app:request-eval}\label{app:update-exp}\label{app:allocation-exp}\label{app:ablations}\label{app:statistics}
We evaluate action-conditioned predictors learned from sampled experience in finite-horizon worlds. Sparse, chain, and grid families each use 16 states, four actions, known support, bounded shared rewards, and horizons $H\in\{4,8,12\}$. Forty development worlds fix the protocol, 80 independent worlds supply the finite-world audit quantities, and 240 held-out worlds are used once for evaluation. World identifiers and random streams are disjoint across these roles. The predictor is the Beta-smoothed empirical transition kernel fitted from nested row observations; a separately acquired 32-query-per-row model defines the fixed reference behavior. Every gate sees only the learned predictor and its permitted audit evidence. Exact simulator returns are revealed only after routing and are never used to choose a proposal, a gate, or an evidence location.

The protocol has three complementary components. First, 400 paired constructions share the learned model and passive record but differ on one unobserved action, isolating whether intervention evidence identifies the otherwise hidden world effect. Second, held-out planning and teaching requests compare unconditional positive-imagination acceptance, uncertainty-only rejection, and decision-relative qualification; the same requests measure value error, update-direction cosine, harmful-use rate, and realized gain. Third, online runs compare ungated, uniformly gated, and decision-directed evidence acquisition at an equal real-transition budget. Candidate proposals, reference behavior, and evaluation streams are matched within every comparison, so the only manipulated object is whether the world-model prediction is qualified and where new evidence is acquired.

Two theory-directed supplements reuse the same finite-world row-audit primitive. For evidence-complexity scaling, the audited Bernoulli row has means $1/2-\gamma$ and $1/2+\gamma$ in the two possible worlds, with 13 logarithmically spaced margins $\gamma\in[0.02,0.20]$. At each margin we enumerate sample sizes and report the smallest $n_{0.9}$ for which the exact randomized Neyman--Pearson test reaches $90\%$ power while controlling false promotion at $\alpha=0.05$; the log--log slope is fitted once across all registered margins. For evidence reuse, we fix $\gamma=0.08$, audit all eight transition rows, and evaluate $K\in\{10,20,50,100,200\}$ later comparisons. The shared rule pays for one simultaneous $\alpha$-level audit; the re-audit baseline pays for $K$ fresh tests, each assigned error $\alpha/K$ and separately sized to the same $90\%$ power. Consequently both procedures control familywise false promotion, while their real-query costs are measured under identical accuracy requirements. Exact integer results and the complete plotting records are retained with the experiment archive.

\paragraph{Theory-to-measurement correspondence.}
The finite-world design gives each theoretical statement an observable counterpart. Paired worlds implement passive non-identifiability and interventional recovery (Theorem~\ref{thm:impossibility} and Proposition~\ref{prop:attribution-sampling}); attribution accuracy and model-effect error are evaluated before any decision gate is scored. Planning curves test decision reliability and the separation between predictive accuracy and decision safety (Theorem~\ref{thm:planning} and Corollary~\ref{cor:absolute}) through realized gain, harmful-use rate, and risk--coverage. Online runs instantiate the closed-loop and adaptive-acquisition statements (Theorem~\ref{thm:closed-loop} and Theorem~\ref{thm:sequential}), while gradient and teaching diagnostics test gradient reliability and local safe improvement (Theorem~\ref{thm:gradient} and Corollary~\ref{cor:safe-update}). The log--log experiment tests the evidence lower-bound order (Proposition~\ref{prop:verification-lower}), and the reuse experiment tests amortized verification and precedence accounting (Proposition~\ref{prop:amortized-cost} and Theorem~\ref{thm:allocation}). This one-to-one mapping makes the appendix a validation suite for the theory rather than a collection of unrelated plots.

\subsection{Results and supplementary summaries}
\label{app:paired-exp}\label{app:audit-exp}\label{app:neural-exp}\label{app:extended-plots}
The results separate the three theoretical claims cleanly. The paired construction reaches $100\%$ attribution accuracy at the largest intervention budget, while model-effect error decreases from $0.08999$ to $0.01888$ (Figure~\ref{fig:app-exact}). On held-out planning requests, accepting every positive imagined advantage produces $5.10\%$ harmful uses; decision-relative qualification reduces this to $0.30\%$ while preserving positive aggregate gain (Figures~\ref{fig:app-planning}--\ref{fig:app-data}). The update-direction supplement raises mean true/imagined gradient cosine from $0.761$ to $0.9994$ as evidence increases. In the online study, ungated, uniform-gated, and decision-directed acquisition incur 44, 10, and 4 harmful updates, respectively (Figures~\ref{fig:app-bounds}--\ref{fig:app-secondary}). The evidence-complexity fit has slope $2.014$ on the log--log scale, matching the predicted $\gamma^{-2}$ order. At $K=200$ downstream comparisons, one shared audit uses 664 real queries versus 350,400 for fresh re-auditing, a $527.7\times$ reduction under the same registered familywise error and power targets.

These conclusions do not rely on post-hoc selection: all registered worlds and margins are retained, gates never observe exact evaluation returns, and paired methods reuse identical stochastic streams. The simultaneous theorem-derived certificate can abstain completely at small budgets, and adaptive acquisition is not uniformly superior in every cell; we report these limitations rather than replacing them with a favorable subset. The experiment therefore checks the finite-world assumptions and scaling laws directly, while the agent benchmark in Appendix~\ref{app:benchmark-eval} separately tests whether the same verify-then-promote structure remains useful without claiming a finite-state certificate.

\begin{figure}[htbp]
\centering\includegraphics[width=\linewidth]{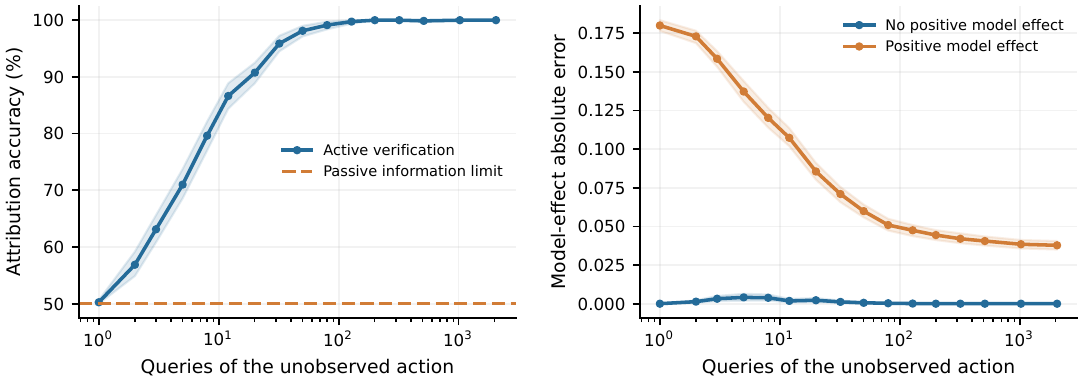}
\caption{Attribution and model-effect error under increasing intervention evidence. Active action-conditioned queries resolve worlds that are indistinguishable under the shared passive record.}
\label{fig:app-exact}
\end{figure}

\begin{figure}[htbp]
\centering
\begin{minipage}{.48\linewidth}
\includegraphics[width=\linewidth]{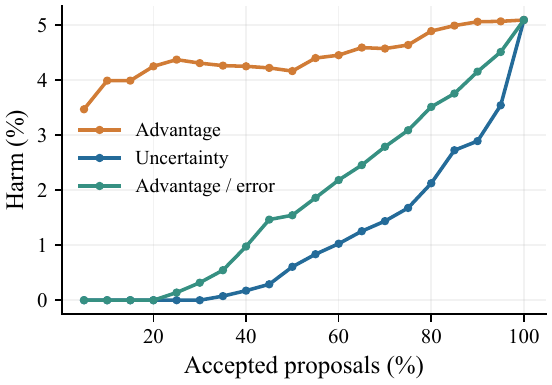}
\end{minipage}\hfill
\begin{minipage}{.48\linewidth}
\includegraphics[width=\linewidth]{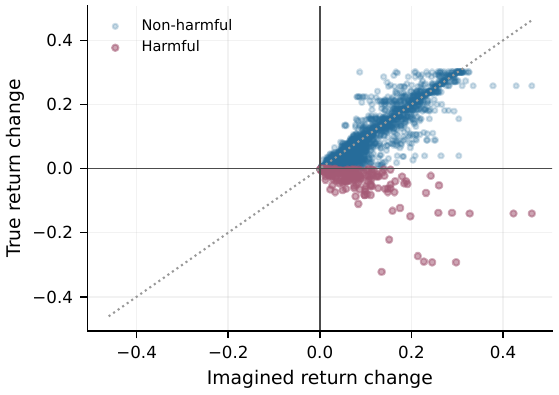}
\end{minipage}
\caption{Decision-relative planning qualification. The risk--coverage curve (left) and imagined-versus-realized changes (right) show that screening the proposed use, rather than trusting positive imagination alone, removes most harmful promotions.}
\label{fig:app-planning}
\end{figure}

\begin{figure}[htbp]
\centering\includegraphics[width=\linewidth]{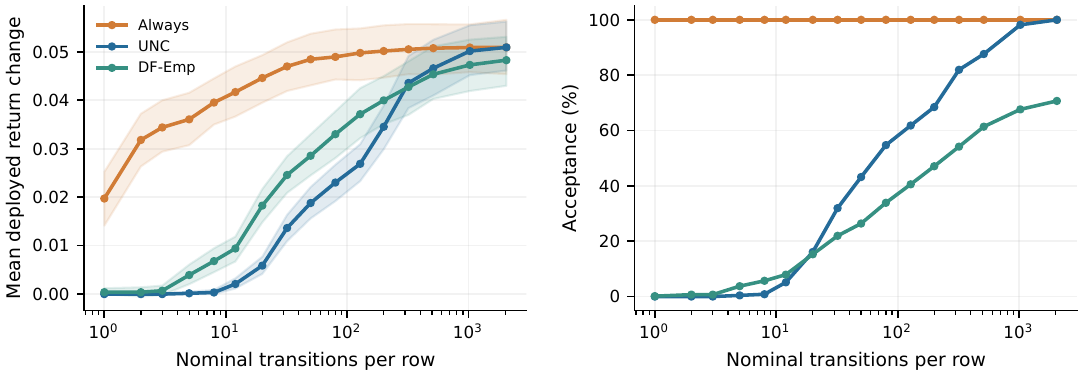}
\caption{Planning gain and promotion rate as action-conditioned evidence increases. Reliability improves without changing the candidate requests or reference policy.}
\label{fig:app-data}
\end{figure}

\begin{figure}[htbp]
\centering\includegraphics[width=\linewidth]{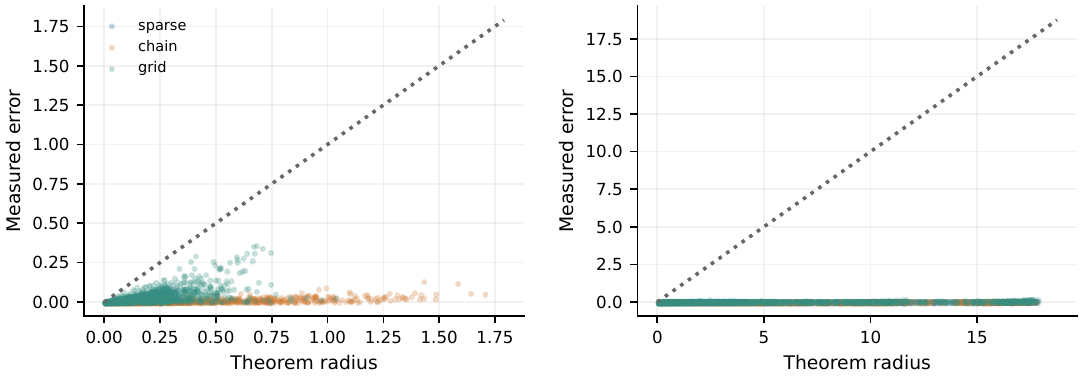}
\caption{Certificate diagnostics. Measured decision errors remain below their theorem-derived radii over the registered finite-world requests.}
\label{fig:app-bounds}
\end{figure}

\begin{figure}[htbp]
\centering\includegraphics[width=\linewidth]{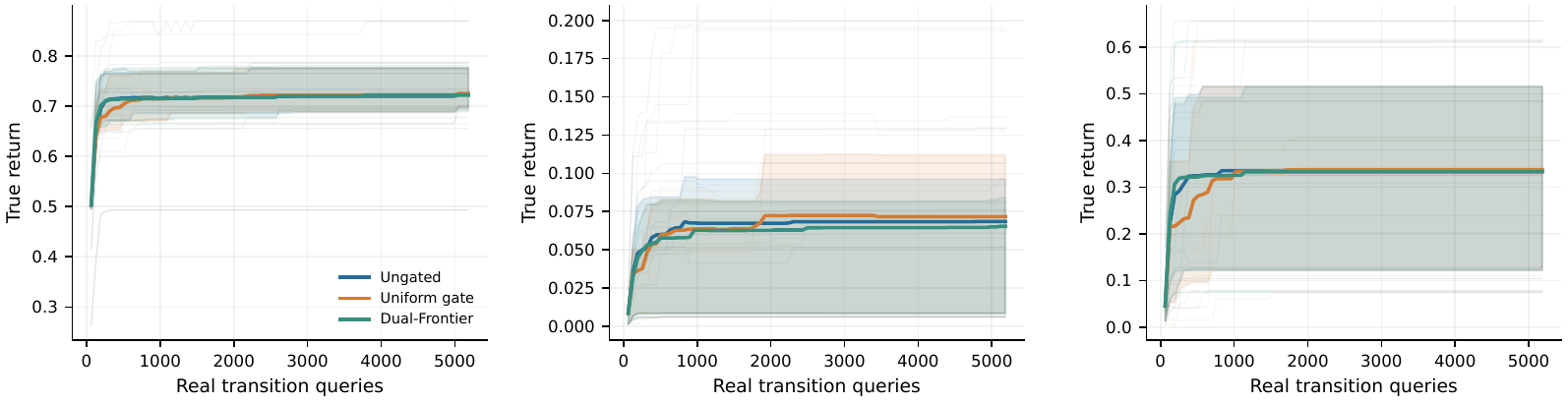}
\caption{Online qualification and acquisition under equal real-transition budgets. Decision-directed evidence reduces harmful updates relative to ungated and uniform alternatives.}
\label{fig:learning}
\end{figure}

\begin{figure}[htbp]
\centering
\begin{minipage}{.32\linewidth}\centering\includegraphics[width=\linewidth,height=1.05in,keepaspectratio=false,viewport=0 0 260 188,clip]{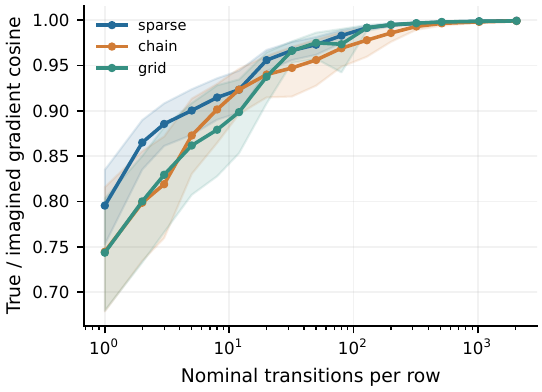}\end{minipage}\hfill
\begin{minipage}{.32\linewidth}\centering\includegraphics[width=\linewidth,height=1.05in,keepaspectratio=false]{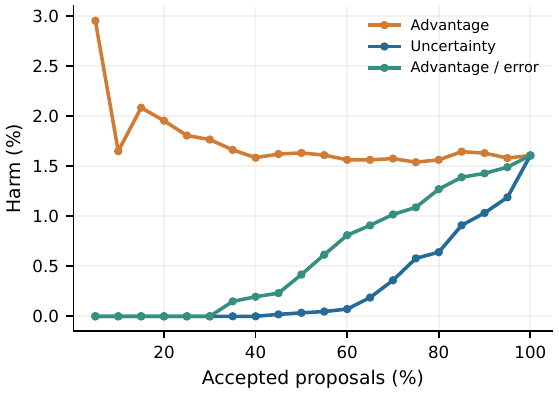}\end{minipage}\hfill
\begin{minipage}{.32\linewidth}\centering\includegraphics[width=\linewidth,height=1.05in,keepaspectratio=false]{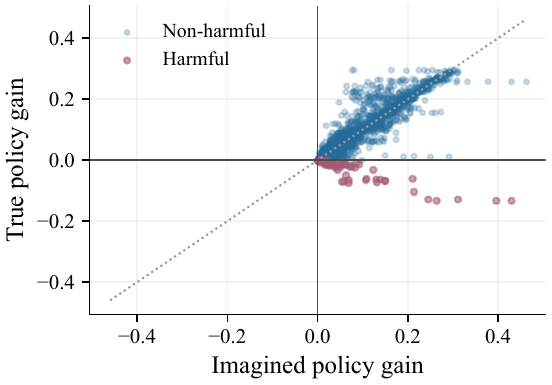}\end{minipage}
\caption{Transfer diagnostics for update-direction alignment (left), teaching risk--coverage (center), and imagined-versus-realized teaching gains (right).}
\label{fig:app-secondary}
\end{figure}

\clearpage
\section{Agent--World-Model Benchmark Evaluation}
\label{app:benchmark-eval}

The benchmark study evaluates the Dual-Frontier admission rule of Section~\ref{sec:dual} under a language-world-model interface. We fix Qwen-AgentWorld as the learned world model, vary the agent backbone between Llama-3.1-8B-Instruct and Qwen3-8B, and evaluate BFCL~\citep{patil2025bfcl}, API-Bank~\citep{li2023apibank}, and NexusRaven~\citep{nexusraven2023}. For each benchmark, the request manifest is partitioned once into a $20\%$ calibration split and an $80\%$ held-out test split. The partition is fixed before calibration and routing evaluation and is shared across all compared methods; the same request partition is used for both agent backbones.

Each request defines a reference--candidate comparison before admission: the base-agent output is the reference and a fixed Qwen-AgentWorld proposal is the candidate. Base actions and world-model generations are produced once and cached, and every non-Agent-only routing rule acts on the same candidate records. Calibration requests are used only to construct the decision-relevant error radius entering the Dual-Frontier margin. Their benchmark references are accessed only after all corresponding model-side predictions have been frozen. Test references are never used in candidate construction, calibration, or admission and are loaded only after the routed test outputs are fixed. Prompts, tool schemas, sampling budgets, parsers, generation order, and candidate-generation records are otherwise shared across routing rules.
\subsection{Decision rules and calibrated Dual-Frontier instantiation}
\label{app:benchmark-rules}
\paragraph{Predicted advantage and estimation uncertainty.}
Once the shared proposal $c_i^K$ is fixed, the world model evaluates its consequence relative to the reference action $c_i^0$ using a fixed repeated-evaluation budget. Let
$\widehat\gamma_{i\ell}\in[-1,1]$, $\ell=1,\ldots,L$, denote the resulting normalized candidate-minus-reference consequence scores. We define
\begin{equation}
S_i=\frac{1}{L}\sum_{\ell=1}^{L}\widehat\gamma_{i\ell},
\qquad
\xi_i=\sqrt{\frac{2\log(1/\delta)}{L}},
\label{eq:benchmark-advantage}
\end{equation}
where the same $L$ and $\delta$ are used throughout evaluation. Thus $S_i$ is the benchmark-side predicted advantage and $\xi_i$ accounts for finite-generation estimation uncertainty.

\paragraph{Decision-relevant model-error calibration.}
The repeated world-model records additionally provide proposal-level reliability information. We collect these quantities in
\begin{equation}
z_i=
\bigl(
1-\bar\rho_i,\,
1-\kappa_i,\,
h_i^{\max}
\bigr),
\qquad
\psi_i=g(z_i),
\label{eq:benchmark-uncertainty-score}
\end{equation}
where $g$ is a fixed nonnegative aggregation rule specified before calibration and used unchanged at test time. The individual confidence, agreement, and consequence-risk signals therefore contribute to the model-error estimate rather than acting as separate Dual-Frontier admission thresholds.

For each calibration request $j$, all model-side quantities are frozen before the benchmark evaluator provides the normalized candidate-minus-reference contrast $\Gamma_j^{\mathrm{eval}}$. We calibrate the residual decision discrepancy not already accounted for by $\xi_j$,
\begin{equation}
e_j^{B}
=
\left[
\left|S_j-\Gamma_j^{\mathrm{eval}}\right|-\xi_j
\right]_+,
\qquad
q_j=e_j^{B}-\psi_j.
\label{eq:benchmark-calibration-residual}
\end{equation}
For a calibration set of size $n_{\mathrm{cal}}$, let
\begin{equation}
k=
\left\lceil
(n_{\mathrm{cal}}+1)(1-\alpha)
\right\rceil,
\qquad
B_i=
\left[\psi_i+q_{(k)}\right]_+,
\label{eq:benchmark-calibrated-B}
\end{equation}
with $q_{(k)}$ the corresponding calibration quantile. Calibration is performed only on the designated $20\%$ split; $B_i$ is then evaluated without test labels on the remaining requests. Under the exchangeability condition of Proposition~\ref{prop:conformal}, the same rank argument gives
\begin{equation}
\Pr\!\left\{e_i^{B}>B_i\right\}\leq\alpha.
\label{eq:benchmark-B-coverage}
\end{equation}

\paragraph{Dual-Frontier promotion rule.}
The benchmark implementation uses the same three-term margin as Eq.~\eqref{eq:decision-margin}:
\begin{equation}
T_i^{\mathrm{DF}}
=
S_i-B_i-\xi_i.
\label{eq:benchmark-df-margin}
\end{equation}
A world-model proposal is promoted only when $T_i^{\mathrm{DF}}>0$. On the calibration-coverage event in~\eqref{eq:benchmark-B-coverage},
\begin{equation}
\Gamma_i^{\mathrm{eval}}
\geq
S_i-B_i-\xi_i
=
T_i^{\mathrm{DF}},
\label{eq:benchmark-bridge}
\end{equation}
so a positive benchmark margin has exactly the lower-bound semantics required by the Dual-Frontier decision rule. This construction separates predicted benefit, decision-relevant model error, and finite-generation estimation uncertainty while keeping the practical gate identical in form to the theoretical one.

\paragraph{Six routing rules.}
Let $a_i^r\in\{0,1\}$ denote whether rule $r$ admits the shared proposal $c_i^K$. If no valid proposal exists, all non-Agent-only rules fall back to $c_i^0$. For the three single-signal comparison rules, we retain fixed thresholds
\begin{equation}
\tau_\rho=0.70,\qquad
\tau_\kappa=3/7,\qquad
\tau_h^{P}=0.22,
\label{eq:benchmark-baseline-thresholds}
\end{equation}
used only to define the corresponding baselines. The routing decisions are
\begin{align}
a_i^{A}&=0,
&
a_i^{W}&=1,
&
a_i^{C}&=\ind\{\bar\rho_i\geq\tau_\rho\},
\nonumber\\
a_i^{K}&=\ind\{\kappa_i\geq\tau_\kappa\},
&
a_i^{P}&=\ind\{h_i^{\max}\leq\tau_h^{P}\},
&
a_i^{D}&=\ind\{T_i^{\mathrm{DF}}>0\}.
\label{eq:benchmark-rules}
\end{align}
Here $A,W,C,K,P,D$ denote Agent-only, Always-WM, Confidence, Consistency, Pessimistic, and Dual-Frontier, respectively. The candidate is identical across $W,C,K,P,D$; only the admission rule changes. The routed output is
\begin{equation}
c_i^r=
\begin{cases}
c_i^K,&a_i^r=1,\\
c_i^0,&a_i^r=0.
\end{cases}
\label{eq:benchmark-routed-output}
\end{equation}

The proposal-generation protocol, uncertainty construction, calibration level, and routing rules are fixed before evaluation of the held-out test split. Calibration references are confined to the designated calibration requests, and test references are inaccessible until~\eqref{eq:benchmark-routed-output} has been frozen. Every held-out request is retained in evaluation, including fallbacks and malformed outputs. Oracle is excluded because it observes reference outcomes and is not deployable.

\subsection{Metrics}
\label{app:benchmark-metrics}

For request $i$ and routing rule $r$, let $c_i^0$ denote the base-agent output and
$c_i^r$ the frozen routed output in~\eqref{eq:benchmark-routed-output}.
We use the fixed benchmark-aware evaluator of each benchmark throughout. Let
$y_i(c)\in\{0,1\}$ denote strict task success, $p_i(c)\in[0,1]$ parameter-level
accuracy, $u_i(c)\in[0,1]$ the decision-quality score assigned to output $c$, and
$q_i(c_i^0,c)\in[0,1]$ the corresponding reliability-loss score relative to the
base action. All evaluator definitions are fixed before routing-rule replay.

Evaluation follows the native structural conventions of each benchmark.
BFCL v4 preserves ordered and parallel-call multiplicity and validates function
names and arguments against the supplied schema. API-Bank matches the required
API identity and its normalized parameter dictionary. NexusRaven canonicalizes
its Python-style function expression before comparing function identity,
argument names, values, and multiplicity. Thus the metrics below share a common
form while retaining the native call semantics of each benchmark.

\paragraph{Task success.}
The strict task-level metric is
\begin{equation}
\mathrm{TS}_r
=
\frac{1}{N}\sum_{i=1}^{N} y_i(c_i^r).
\end{equation}
It requires the complete function-call sequence, including function names,
arguments, and call multiplicity, to satisfy the benchmark reference.

\paragraph{Parameter accuracy.}
We report the average parameter-level score
\begin{equation}
\mathrm{PA}_r
=
\frac{1}{N}\sum_{i=1}^{N} p_i(c_i^r),
\end{equation}
where $p_i(\cdot)$ follows the corresponding benchmark's native parameter
matching and normalization rules.

\paragraph{Net decision gain.}
To measure the signed effect of routing through the learned world model, we use
\begin{equation}
\mathrm{NDG}_r
=
\frac{1}{N}\sum_{i=1}^{N}
\left[
u_i(c_i^r)-u_i(c_i^0)
\right].
\end{equation}
Positive values indicate that the admitted world-model revisions improve the
average decision quality relative to the base agent, whereas negative values
indicate net degradation. We report NDG in percentage points.

\paragraph{Harmful revisions.}
Let $h_i(c_i^0,c_i^r)\in[0,1]$ denote the evaluator's degradation score for the
routed output relative to the base action. We report
\begin{equation}
\mathrm{HRR}_r
=
\frac{1}{N}\sum_{i=1}^{N}
h_i(c_i^0,c_i^r).
\end{equation}
This quantity measures the overall exposure to harmful interventions, including
both their occurrence and decision-level severity.

\paragraph{Revision coverage.}
The fraction of requests on which rule $r$ admits a world-model revision is
\begin{equation}
\mathrm{RC}_r
=
\frac{1}{N}\sum_{i=1}^{N} a_i^r.
\end{equation}
Coverage is descriptive rather than an objective by itself and should be read
jointly with the reliability metrics.

\paragraph{Selective risk.}
Among admitted revisions, we measure the residual reliability loss as
\begin{equation}
\mathrm{SR}_r
=
\frac{
\sum_{i=1}^{N}
a_i^r\, q_i(c_i^0,c_i^r)
}{
\sum_{i=1}^{N} a_i^r
},
\qquad
\sum_i a_i^r>0.
\end{equation}
Selective risk is undefined when a rule never revises. The pair
$(\mathrm{RC},\mathrm{SR})$ therefore provides the empirical risk--coverage
view of the verify-then-promote decision rule.

TS, PA, and NDG are higher-is-better metrics, whereas HRR and SR are
lower-is-better; RC is descriptive. Here $N$ always denotes requests in the
held-out $80\%$ test split; calibration requests are excluded from every
reported benchmark metric. The reported Avg.\ column is the equal-weight
arithmetic mean of the three benchmark percentages rather than a pooled
micro-average.

\subsection{Additional results}
\label{app:benchmark-additional}

The component ablation removes one term at a time from the same Dual-Frontier decision margin while keeping the shared proposal, calibrated quantities, and cached world-model records fixed. Advantage-only admits when $S_i>0$; w/o world-model error uses $S_i-\xi_i>0$; w/o estimation error uses $S_i-B_i>0$; and full Dual-Frontier uses
\[
S_i-B_i-\xi_i>0.
\]
The comparison therefore isolates the contribution of predicted advantage, decision-relevant model error, and finite-generation estimation uncertainty without changing proposal generation or test requests. Table~\ref{tab:agentworld-ablation-qwen} reports the Qwen3-8B results; the corresponding Llama-3.1-8B-Instruct results appear in Table~\ref{tab:agentworld-ablation}.

The complete margin is consistently strongest. Relative to Advantage-only, Dual-Frontier improves the equal-weight Success/Acc. averages by $7.05/4.34$ points for Llama-3.1-8B-Instruct and $6.27/3.92$ points for Qwen3-8B. Removing either uncertainty term recovers only part of this gain, while their joint use yields the strongest performance for both backbones. This pattern is consistent with Section~\ref{sec:dual}: predicted benefit identifies potentially useful interventions, whereas $B_i$ and $\xi_i$ determine whether that apparent advantage remains sufficiently supported for promotion.

\begin{table}[htbp]
\caption{Additional Qwen3-8B ablation results. Success and Acc. denote task success and parameter accuracy; Avg. is the uniform mean over BFCL v4, API-Bank, and NexusRaven.}
\label{tab:agentworld-ablation-qwen}
\vspace{3pt}
\centering
\scriptsize
\setlength{\tabcolsep}{3.25pt}
\renewcommand{\arraystretch}{0.96}
\resizebox{\textwidth}{!}{%
\begin{tabular}{@{}cc*{8}{c}@{}}
\toprule
\addlinespace[2.5pt]
\multirow{2}{*}{Models} & \multirow{2}{*}{Methods}
& \multicolumn{2}{c}{BFCL v4} & \multicolumn{2}{c}{API-Bank}
& \multicolumn{2}{c}{NexusRaven} & \multicolumn{2}{c}{Avg.}\\
\cmidrule(lr){3-4}\cmidrule(lr){5-6}\cmidrule(lr){7-8}\cmidrule(l){9-10}
& & Success & Acc. & Success & Acc. & Success & Acc. & Success & Acc.\\
\midrule
\multirow{4}{*}{Qwen3-8B}
& Advantage-only         & 67.80 & 84.94 & 95.07 & 96.41 & 94.40 & 95.71 & 85.76 & 92.35\\
& w/o world-model error  & 72.32 & 88.01 & 96.88 & 97.50 & 96.02 & 96.68 & 88.41 & 94.06\\
& w/o estimation error   & 77.79 & 91.18 & 98.40 & 98.31 & 97.11 & 97.52 & 91.10 & 95.67\\
& \textbf{Dual-Frontier} & \textbf{79.38} & \textbf{92.20} & \textbf{98.92} & \textbf{98.64} & \textbf{97.80} & \textbf{97.96} & \textbf{92.03} & \textbf{96.27}\\
\bottomrule
\end{tabular}
}
\end{table}

\paragraph{Coverage and conditional reliability.}
Tables~\ref{tab:revision-coverage} and~\ref{tab:selective-risk} report how often each deployable rule admits a world-model revision and the residual risk among those admitted revisions. Agent-only is omitted because it never revises; consequently, its selective risk is undefined rather than zero. Revision coverage is descriptive, whereas lower selective risk is better. Their joint reading is essential: Dual-Frontier deliberately operates at moderate coverage while removing most unsafe promotions, exactly the selective qualification behavior predicted by the theory.

Across the three benchmarks, Dual-Frontier revises $57.59\%$ of Llama and $57.51\%$ of Qwen requests. At this nontrivial coverage, its average selective risk is $28.20\%$ and $18.23\%$, respectively, versus $67.68\%$ and $78.20\%$ for the strongest competing selective baseline. The reductions of $39.48$ and $59.97$ percentage points explain why the primary-metric gains are not a consequence of indiscriminate revision: the method rejects precisely the candidate groups most likely to erase a correct base decision. Always-WM provides the opposite endpoint, with full coverage but selective risks of $78.89\%$ and $86.29\%$.

\par\medskip
\noindent\begin{minipage}{\textwidth}
\refstepcounter{table}\label{tab:revision-coverage}
\raggedright Table~\thetable: Revision coverage (\%) across agent backbones and benchmarks. Avg. is the uniform mean over BFCL v4, API-Bank, and NexusRaven.\par
\vspace{4pt}
\centering
\small
\renewcommand{\arraystretch}{0.96}
\begin{tabular*}{\textwidth}{@{\extracolsep{\fill}}llrrrr@{}}
\toprule
Models & Methods & BFCL v4 & API-Bank & NexusRaven & Avg.\\
\midrule
\multirow{5}{*}{\shortstack[c]{Llama-3.1-8B\\Instruct}}
& Always-WM     & 100.00 & 100.00 & 100.00 & 100.00\\
& Confidence    &  99.50 &  99.61 &  98.74 &  99.28\\
& Consistency   &  80.88 &  84.65 &  83.33 &  82.95\\
& Pessimistic   &  82.25 &  85.43 &  79.87 &  82.52\\
& Dual-Frontier &  55.62 &  61.81 &  55.35 &  57.59\\
\midrule
\multirow{5}{*}{Qwen3-8B}
& Always-WM     & 100.00 & 100.00 & 100.00 & 100.00\\
& Confidence    &  99.62 &  99.61 &  99.69 &  99.64\\
& Consistency   &  78.25 &  81.30 &  81.45 &  80.33\\
& Pessimistic   &  82.50 &  79.53 &  84.91 &  82.31\\
& Dual-Frontier &  53.62 &  57.28 &  61.64 &  57.51\\
\bottomrule
\end{tabular*}
\end{minipage}

\par\medskip
\noindent\begin{minipage}{\textwidth}
\refstepcounter{table}\label{tab:selective-risk}
\raggedright Table~\thetable: Selective risk (\%) among admitted world-model revisions. Lower is better; Avg. is the uniform mean over the three benchmarks.\par
\vspace{4pt}
\centering
\small
\renewcommand{\arraystretch}{0.96}
\begin{tabular*}{\textwidth}{@{\extracolsep{\fill}}llrrrr@{}}
\toprule
Models & Methods & BFCL v4 & API-Bank & NexusRaven & Avg.\\
\midrule
\multirow{5}{*}{\shortstack[c]{Llama-3.1-8B\\Instruct}}
& Always-WM     & 93.12 & 73.43 & 70.13 & 78.89\\
& Confidence    & 90.95 & 66.21 & 57.64 & 71.60\\
& Consistency   & 89.49 & 61.86 & 51.70 & 67.68\\
& Pessimistic   & 90.43 & 65.67 & 56.69 & 70.93\\
& \textbf{Dual-Frontier} & \textbf{9.12} & \textbf{33.66} & \textbf{41.82} & \textbf{28.20}\\
\midrule
\multirow{5}{*}{Qwen3-8B}
& Always-WM     & 97.88 & 79.53 & 81.45 & 86.29\\
& Confidence    & 96.99 & 72.53 & 76.03 & 81.85\\
& Consistency   & 96.49 & 67.07 & 71.04 & 78.20\\
& Pessimistic   & 96.97 & 72.52 & 75.56 & 81.68\\
& \textbf{Dual-Frontier} & \textbf{3.43} & \textbf{27.36} & \textbf{23.90} & \textbf{18.23}\\
\bottomrule
\end{tabular*}
\end{minipage}

\clearpage

\section*{AI Use Statement}
Generative AI tools were used solely for formatting checks and language polishing. The authors
reviewed all resulting revisions and take full responsibility for the final manuscript.

\end{document}